\documentclass{llncs}
\usepackage{times}
\usepackage{mathtools} 
\usepackage[margin=1.76in]{geometry}  
\usepackage{changepage}
\usepackage{enumitem} 
\usepackage{xcolor}
\usepackage{url} 
\usepackage{todonotes}
\usepackage{amssymb}

\newcommand{\equalTo}{\textsf{Eq}}

\newcommand{\getAttack}{\textsf{Attck}}

\newcommand{\getSupport}{\textsf{Spprt}} 
\newcommand{\BBA}{\mbox{\textsf{BBA}}}  
\newcommand{\F}{\ensuremath{\mathcal{A}^{\small \BBA}}}  
\newcommand{\Fe}[1]{\ensuremath{a_{#1}^{\small \BBA}}} 
 
\newcommand{\ABA}{\mbox{\textsf{ABA}}}
\newcommand{\ASPIC}{\mbox{\textsf{ASPIC}}}

\newcommand{\inL}{\textsf{in}}
\newcommand{\outL}{\textsf{out}}
\newcommand{\undecL}{\textsf{undec}}
  
\newcommand{\Rs}{\ensuremath{R^{\mathbf{s}}}}

\newcommand{\ang}[1]{\ensuremath{\langle #1 \rangle}}

\newcommand{\hide}[1]{}

\newcommand{\andC}{\textsf{and}}

\begin{document}
     \title{Block Argumentation} 
%\titlenote{Produces the permission block, and
%  copyright information}
%\subtitle{Extended Abstract}
%\subtitlenote{The full version of the author's guide is available as
%  \texttt{acmart.pdf} document}

\author{} 
\author{Ryuta Arisaka, Stefano Bistarelli \and Francesco Santini}

    \institute{Department of Mathematics and Computer Science,
    University of Perugia, Italy
     \email{ryutaarisaka@gmail.com, firstname.lastname@unipg.it}
}

    \maketitle
\vspace{-0.2cm} 
    
\begin{abstract}   
  %Non-logical elements, such as emotion, 
  %preconception, personal qualities or  
  %pursuit of self-interest, play significant parts 
  %in determining the outcome of real-life argumentation. 
  %There are at least two 
  %key traits worth mentioning of 
  % the non-logical elements: that, firstly,  
  %they are often triggered episodically: 
  %ill-deeds or merits of an agent 
  %tend to be introduced through past episodes;
  %and that, secondly, they often carry with them side-effects 
  %that may restrain or elicit certain arguments 
  %of agents'. We introduce block argumentation theory in this work
  %for treating them in an orderly manner. It is based on 
  %bipolar abstract argumentation with attacks and supports. 
  %%Every coherent argumentation (block) is treated an argument. 
  %Side-effects are issued only by blocks, which 
  %other blocks may counter by arguing against them. 
  %We illustrate the theory with legal examples.
  \vspace{-0.2cm}
  We contemplate a higher-level bipolar 
  abstract argumentation 
  for 
  non-elementary arguments such as: X argues against 
  Y's sincerity with the fact that 
  Y has presented his argument to draw a conclusion C, 
  by omitting other facts which would not have validated C.  
  Argumentation involving such arguments  
  requires us to potentially consider an argument as a coherent block
  of argumentation, i.e. an argument may itself be 
  an argumentation. %Hence %where is the boundary between 
  %an argument and an argumentation? There is no boundary; 
  %there is no boundary between an argument and an argumentation.
  %only a perspective distinguishes the two, as 
  %#with the Rabbit-Duck illusion. 
  In this work, we formulate \emph{block argumentation} 
  as a specific instance of Dung-style bipolar abstract argumentation
  with the dual nature
%   \todo{Self-similarity is neither explained in the abstract, nor in the introduction. While a reader can understand ``block'', there is no way to understand what self-similarity is.}
  of  
  arguments. We consider internal consistency 
  of an argument(ation) under a set of constraints, of graphical (syntactic) and of 
  semantic nature, and  
  formulate acceptability semantics in relation to them.  
  We discover that classical acceptability semantics do not 
  in general hold good with the constraints. In particular,  
  acceptability of unattacked arguments is not always warranted. Further, there may not be 
  a unique minimal member in complete semantics, 
  thus sceptic (grounded) semantics
  may not be its subset. To retain 
  set-theoretically minimal semantics 
  as a subset of complete semantics, 
  we define semi-grounded semantics. 
  Through comparisons, we show how the concept of block argumentation 
  may further generalise structured argumentation. %as may be handled 
   %in block argumentation. 
  \end{abstract} 

\vspace{-0.9cm}
\section{Introduction}       
In higher-level argumentation \cite{Gabbay09b}, 
or temporal/modal argumentation networks \cite{Barringer12}, an argumentation $F_1$ may be substituted into an argument of a given argumentation $F_2$. 
In the case of higher-level argumentation, the key is to find out which arguments in $F_1$ may be attacking or attacked by other arguments in $F_2$. 
In meta-level argumentation \cite{Modgil11,Villata10}, 
properties of argumentation at object level 
- such as whether an object-level argument 
attacks another object-level argument, or 
the trustworthiness of an arguer -
may be discussed in  an argumentation about the object-level argumentation.

In a realistic argumentation, 
it can of course happen that some agent 
argues about some attack of an argument $a_1$ on an argument $a_2$. However, that 
may not be at meta-level; the agent's argument could be 
in the same argumentation with $a_1$ and $a_2$, interacting with any arguments in the argumentation. Thus, the clear-cut 
distinction between object-level and meta-level argumentations does not always
apply. Further,  
% an argumentation itself can appear as an argument that
% attacks an argument: then  
% one has to look into the local argumentation and 
% identify which of its arguments 
% is the actual attacker/attackee. 
it may be that some argumentation $F$ itself as an argument $a_F$ attacks
an argument, in which case 
it is not the case that we must seek the origin of 
the attack in arguments in $F$, which does not conform to the principle 
of higher-level argumentation. 

Seeing the gap, in this work, we formulate \emph{block argumentation}, where 
an argumentation 
may be an argument and vice-versa, for 
non-elementary arguments 
such as: X argues against 
Y's sincerity with the fact that 
  Y has presented his argument to draw a conclusion C, 
  by omitting other facts which would not have validated C. 
 This  argument is in the form: ``The argument:[the argument:[other facts attack C] attacks the argument:[some facts support C]] attacks Y's sincerity.'', 
indicating an argument may be a block of argumentation. \\
 \indent Having such an argument blurs 
the boundary between an argumentation and an argument, 
leaving the differentiation only up to 
the perspective one employs. 
While block argumentation as we show is a specific instance of 
Dung-style bipolar abstract argumentation (see practical motivation 
for bipolar argumentation in, e.g. \cite{Cayrol11,Arisaka16d,hunter18}) with the dual nature
in arguments, 
it propels us 
to contemplate internal consistency
of an argument. 
For example, one may (or 
may not) find 
\emph{``The conclusion by X in the past that Y was a terrorist 
corroborates X's sinister personality''} factually inconsistent if 
there was no such conclusion by X.\footnote{A possible-world-semantic
observation that some arguments may be unusable was made~\cite{Barringer12}.} 
Ergo, we posit a set of constraints, of graphical (syntactic) and 
semantic nature, which one may choose to impose on a block argumentation, 
and we characterise its acceptability semantics in accordance with them. 
% We design block argumentation in the framework of  bipolar schemes~\cite{Cayrol11}: the motivation is that the classical notion of defence~\cite{Dung95} does not necessarily account for all of the positive relations among arguments~\cite{hunter18}. \ryuta{this reference is?}
As we are to show, classical Dung acceptability semantics do not 
in general hold good once these constraints are taken into consideration.
In particular, 
unattacked arguments may not be outright acceptable. Further, there may be 
more than one minimal member in complete semantics,
thus the sceptic (grounded) semantics 
may not form its subset; such situation 
has already arisen in the context of 
weighted argumentation \cite{lpnmr16,bistarelli18}. To retain set-theoretically minimal semantics as a subset of complete semantics, we define semi-grounded semantics.

The paper
 has the following structure: in Section~\ref{sec:motivation} we motivate our approach with a real legal example from a popular case. Section~\ref{sec:tech} reports the necessary technical preliminaries. Section~\ref{sec:block} presents block (bipolar) argumentation with formal results (e.g. on the existence of semantics). Finally in Section~\ref{sec:related}, we wrap up the paper with related work, where we discuss 
\ABA-style structured argumentation \cite{Dung09}, noting how the concept 
of block argumentation may further generalise the formalism.

% Finally, we also detail 
% connection to structured argumentation, noting what additions 
% to \ASPIC/\ABA-style argumentation \cite{Modgil13,Dung09} 
% would be needed 
% for those non-elementary arguments handled in block argumentation. \todo{self-similarity is in the title: it maybe deserves two words in the intro as well}
\vspace{-0.2cm}
\section{Motivation for Block Argumentation} \label{sec:motivation} 
During the still ongoing trial
 over the death of Kim Jong-Nam\footnote{\url{https://en.wikipedia.org/wiki/Assassination_of_Kim_Jong-nam}.}, a certain 
argumentation was deployed by a suspect's defence lawyer as 
he cast a blame on Malaysian authorities for 
having released only portions of CCTV footage of the fatal attack.
Broadly: 
\vspace{-0.1cm}
\begin{description} 
   \item[Prosecutor:] the CCTV footage released by Malaysian Police 
     shows a suspect walking quickly to an airport restroom
     to wash hands 
     after attacking the victim with VX, 
      which produces an impression that the suspect,
    contrary to her own statement that she thought she 
    was acting for a prank video,  
      knew
     what was on her hands. 
   \item[Defence Lawyer:] however, the CCTV footage in its entirety shows 
    the suspect adjusting her glasses after the attack, with
    VX on her hands, which counter-evidences 
    her knowledge of the substance. Since Malaysian authorities  
    know of the omitted footage, they are clearly 
    biased against the suspect, 
    intentionally tampering with evidence. 
\end{description} 

\noindent Assume the following arguments (with attacks and supports): 

\begin{description} 
  \item[$a_1$:] After the victim was attacked
      with VX, the suspect walked quickly to a restroom 
  for washing hands.
  \item[$a_2$:] The suspect knew VX was on her hands. 
  \item[$a_3$:] The suspect was 
     acting for a prank video.  
  \item[$a_4$:] The suspect adjusted her glasses 
       with VX on her hands before walking to restroom. 
  \item[$a_5$:] Malaysian authorities are biased against the suspect, 
    tampering with evidence by intentional 
     omission of relevant CCTV footage. 
  \item[$a_6$:] $a_1$ supports $a_2$. \quad $a_7${\bf:} \  $a_4$ attacks $a_2$. \quad $a_8${\bf:} \ $a_7$ attacks $a_6$.
  %\item[$a_7$:] $a_4$ attacks $a_2$.  
  %\item[$a_8$:] $a_7$ attacks $a_6$. 
\end{description} 
%\begin{figure}[t]
\begin{center} 
    \includegraphics[scale=0.11]{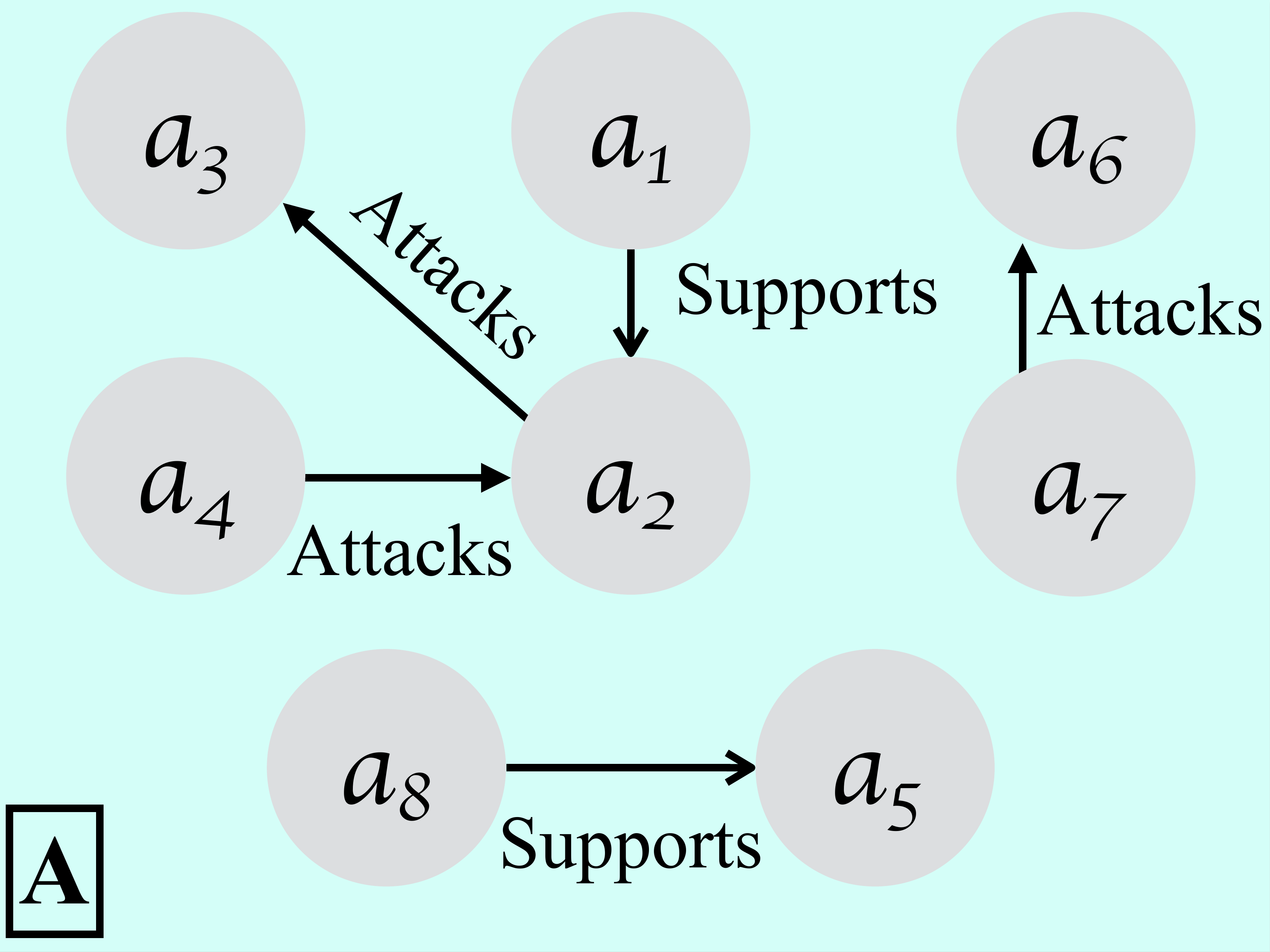}
    \includegraphics[scale=0.11]{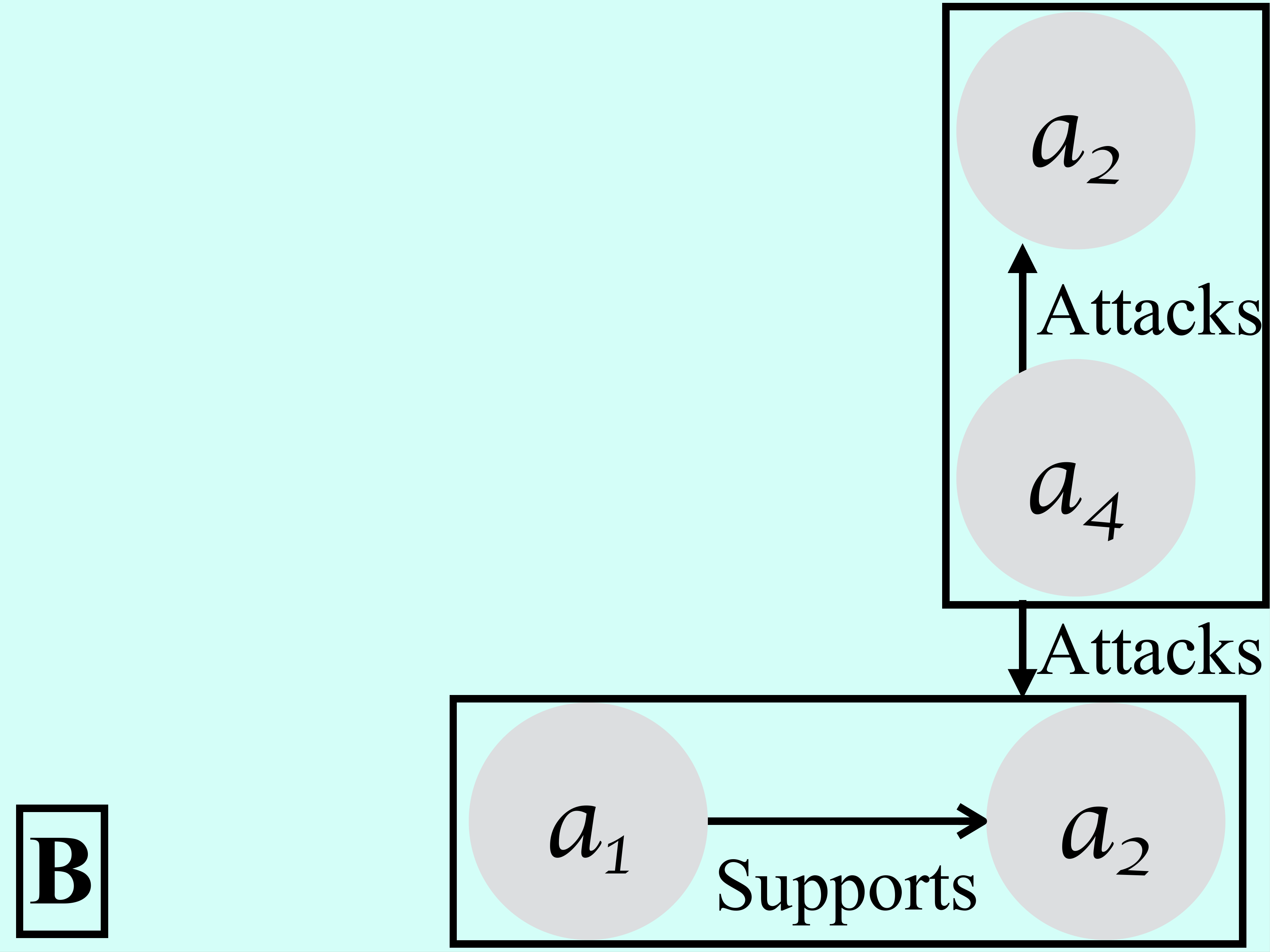}  
   \includegraphics[scale=0.11]{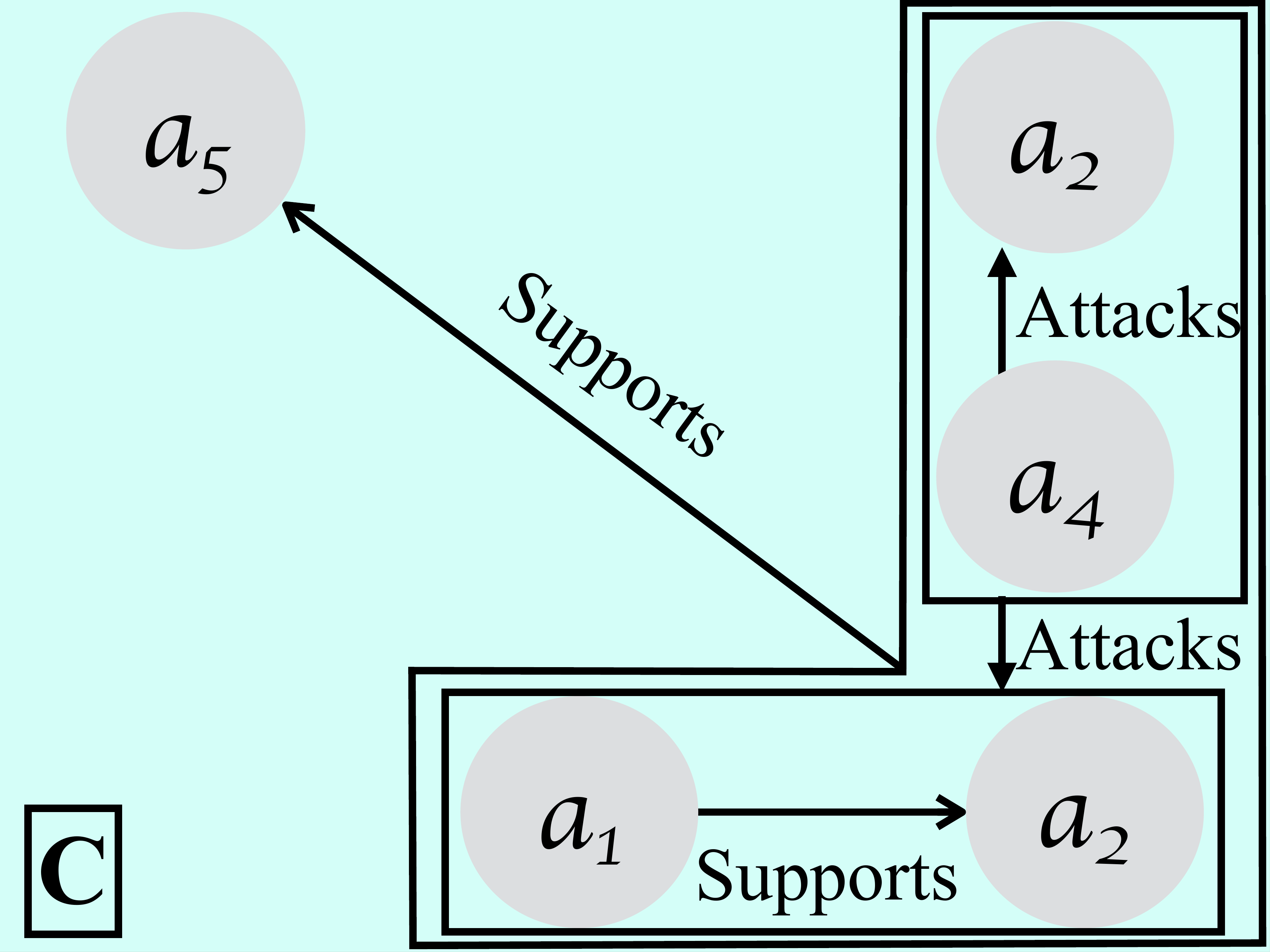}  
\end{center} 
% \caption[solution]{\fbox{A}: Argumentation at court by the Malaysian authorities and 
% a defence lawyer. \fbox{B}: Expansion of ``$a_7$ attacks $a_6$''. 
% \fbox{C}: Expansion of ``$a_8$ supports $a_5$''. }\label{figure:example1}
% \end{figure}
Then we can model the example argumentation as in \fbox{A}. 
%(Fig.~\ref{figure:example1}).  
Malaysian Police uses $a_1$ 
for $a_2$ ($a_1$ supports $a_2$) to dismiss $a_3$ ($a_2$ attacks $a_3$). All these 
three arguments are made available to the audience. The defence lawyer 
uses $a_4$ to counter $a_2$. $a_4$ is also available to the audience as attacking 
$a_2$. He then 
uses $a_7$, which is itself an argumentation, to attack Malaysian Police' argumentation 
$a_6$. This is also presented to the audience. Finally, 
he uses $a_8$, an argumentation, for $a_5$.

%These give us the following 
%sequence 
%of monotonically populating argumentation: \fbox{A} $\rightarrow$ \fbox{B} 
%$\rightarrow \cdots \rightarrow$ \fbox{A}.

% \begin{figure}[!h]  
% \begin{center}
%   \includegraphics[scale=0.14]{assassinationFigure1000.pdf}  
%   \includegraphics[scale=0.14]{assassinationFigure2000.pdf}  
% \end{center} 
% \caption{\fbox{B}: Expansion of \fbox{D}. \fbox{C}: 
% Expansion of ``$a_8$ supports $a_5$''.}
% \end{figure} 

Arguments of the kinds of $a_6$, $a_7$ and $a_8$ 
are themselves argumentations, so 
``$a_7$ attacks $a_6$'' could be detailed as in \fbox{B}, 
and ``$a_8$ supports $a_5$'' as in
\fbox{C}. 
\vspace{-0.4cm}
\subsubsection{Constraints.} These non-elementary arguments
occur in natural language constructs such as ``The fact: [something criticises (supports)
something] is at odds with (in support of) something'', ``The sentiment: [something 
is being defeated by something] enforces something'', and so on, which thus appear 
rather commonly in practice. By recognising these arguments, however, we seem to face 
some challenges to the semantic account given in Dung classic abstract argumentation theory \cite{Dung95}. Assume, for instance, that argumentation as in \fbox{B}: ``$a_7$ attacks $a_6$'', i.e. ``The fact: [the suspect adjusted her glasses with 
VX on her hands before walking to restroom counter-evidences her knowledge of the substance 
on her hands] attacks the argument: [she knew VX was on her hands 
because, after attacking the victim with VX, she walked quickly to a restroom for washing 
hands].'', is given without 
any prior mention of arguments $a_1$, $a_2$ and $a_4$. Then, 
we do not even know what relations may hold between them. Regardless, 
any argument not attacked is always justified in classic theory, according 
to which $a_7$, i.e. that $a_4$ attacks $a_2$, is an acceptable argument, while in the 
first place nothing is known about $a_2$ and $a_4$ that would enable 
us to see $a_4$'s attack on $a_2$ \cite{Slife95}. Essentially, the problem boils down to 
one's interpretation of acceptance of an argument. When we accept an argument, are we only accepting 
the fact that it has resisted refutations, or are we also accepting the information in the 
argument?

On many occasions, acceptance of an argument means the latter, for which the classic prediction could appear hasty and possibly unsafe, more so in block 
argumentation, for an argument can be explicitly seen appearing multiple times, and one may like 
to impose certain constraints to enforce argument's dependency on the same argument 
that occurs in a given argumentation. For \fbox{B}, the constraint to be imposed 
could be 
the presence of the argumentation involving $a_1, a_2$ and $a_4$ outside the blocks, 
as in \fbox{A}. With it, it is immediate whether the argumentation in 
the argument $a_6$ or $a_7$ actually refers to a part of an already known argumentation. 
%\todo{The following two paragraphs can be dropped: technical detail on constraints can go in Section~\ref{sec:block}} 
We envision both graphical (syntactic) and semantic constraints in this paper (more details are provided in Section~\ref{sec:block}).

\vspace{-0.3cm} 

\section{Technical Preliminaries}\label{sec:tech}  
\vspace{-0.1cm}
Let $\mathcal{A}$ be a class of abstract entities we understand 
as arguments. We refer to a member of $\mathcal{A}$ 
by $a$ with or without a subscript and/or a superscript, 
and a finite subset of $\mathcal{A}$ by $A$ with or without 
a subscript.  
A bipolar argumentation framework (e.g. see \cite{Cayrol11})
is a tuple $(A, R, \Rs)$ with two binary relations 
$R$ and $\Rs$ over $A$. For any $(A, R, \Rs)$, 
$A_1 \subseteq A$ is said to attack, or support, $A_2 \subseteq A$ 
if and only if, or iff, there exist $a_1 \in A_1$ and 
$a_2 \in A_2$ such that 
$(a_1, a_2) \in R$ (for attack), or $(a_1, a_2) \in \Rs$ 
(for support). \\  
\indent An extension-based acceptability semantics 
of $(A, R, \Rs)$ is a family of $A$ (i.e. 
a subset of power set of $A$). When $\Rs$ is not taken into account, 
a semantics of $(A, R, \Rs)$ is effectively that of $(A, R)$, 
a Dung abstract argumentation framework \cite{Dung95}. In Dung's, 
$A_1 \subseteq A$ is said to defend $a_x \in A$ 
iff each $a_y \in A$ attacking $a_x$ 
is attacked by at least one member of $A_1$, 
and said to be conflict-free
iff $A_1$ does not attack $A_1$. $A_1 \subseteq A$ 
is said to be: admissible iff it is conflict-free and defended; 
complete iff it is admissible and includes all arguments 
it defends; preferred iff it is a maximal 
complete set; and grounded iff it is the set intersection of 
all complete sets. 
Complete / preferred / grounded semantics is the set of 
all complete / preferred / grounded sets. The grounded set 
defined as the set intersection may not be a complete set \cite{lpnmr16} in non-classic setting; it, 
however, reflects the attitude of sceptic acceptance, the whole point of the grounded 
semantics, which would not be necessarily fulfilled if the grounded set were defined 
the least complete set. 

A label-based acceptability semantics \cite{Caminada06} for Dung's $(A, R)$ makes use of 
the set $\mathcal{L}$ of three elements, say $\{+, -, ?\}$, and the class $\Lambda$ of 
all functions from $\mathcal{A}$ to $\mathcal{L}$. While, normally, 
it is $\{\inL, \outL, \undecL\}$, by $\{+, -, ?\}$ we avoid direct 
acceptability readings off them. $\lambda \in \Lambda$ is said to be a complete 
labelling of $(A, R)$ iff: 
\vspace{-0.2cm}
\begin{itemize}
    \item $\lambda(a) = +$ iff $\lambda(a_p) = -$ for every $a_p \in A$ that attacks $a$.
    \item $\lambda(a) = -$ iff there is some $a_p \in A$ with $\lambda(a_p) = +$ 
     that attacks $a$. 
\end{itemize} 
\vspace{-0.2cm}
$A_1 \subseteq A$ is the set of all 
arguments that map into $+$ under complete labelling iff $A_1$ is a complete set, thus 
a label-based semantics provides the same information as an extension-based 
semantics does, and more because it classifies the remaining arguments 
in $-$ and ?. \\
\indent In bipolar argumentation where $\Rs$ properly matters 
for an acceptability semantics, 
the notion of support is given a particular interpretation  
which influences the semantics.\footnote{See \cite{Cayrol11} for a survey 
of three 
popular types: deductive support \cite{Boella10} (for any member of 
a semantics, if $(a_1, a_2) \in \Rs$ and if $a_1$ is 
in the member, then $a_2$ is also in the member); evidential support \cite{Oren10,Arisaka16d} (for any member of a semantics, 
if $a$ is in the member, then it can be traced back through $\Rs$ to some 
indisputable arguments in the member) and necessary support \cite{Nouioua10,Nouioua11} (for any member of a semantics, 
if $(a_1, a_2) \in \Rs$ and if $a_2$ is in the member, 
then so must $a_1$ also be).} $\Rs$ can 
be used also for the purpose of expressing premise-conclusion relation 
in structured argumentation \cite{Modgil13,Dung09} among arguments, which will 
be separately discussed in Section~\ref{sec:related}.

% \subsection{Domain of discourse and interpretation} 
% In formal logic (see \cite{Kleene52} or any other standard text for the foundation of mathematical logic), 
% there is always the issue of knowing the object(s) a formal symbol actually refers to. 
% In any first-order expression $p(c_1, c_2)$ for a predicate $p$ taking two constants
% $c_1$ and $c_2$, we need to know what each of them signifies in the context of 
% a given domain of discourse that contains all objects of some kinds 
% the symbols could refer to. Interpretation is a function that does the  
% mapping of the symbols to the objects (alternatively the other way round). 
% In this example, $p$ is interpreted as a set of pairs of objects in the domain 
% of discourse, and $c_1$ and $c_2$ are each interpreted as an object. 
% The expression $p(c_1, c_2)$ is true under an interpretation just when 
% the pair of objects as referred to by $c_1$ and $c_2$ happen to belong to 
% the set referred to by $p$. \\
% \indent Our argumentation theory, to be presented in Section 3, is independent of any particular logic; however, it seems 
% impossible to meaningfully discuss
% semantic constraints without any ability to know equality of two arguments. Therefore, rather than 
% trying to circumvent it somehow, we will simply rely upon the notion of 
% interpretation where relevant. 

\section{Block (Bipolar) Argumentation}\label{sec:block} 
Let $\mathbb{N}$ be the class of natural numbers including 0. 
We refer to its member by $n$ with or without a subscript.
Let $\mathcal{X}$ be a class of an uncountable number of abstract entities. 
We refer to a member of $\mathcal{X}$ by $x$ with or without 
a subscript. 
It will be assumed that every member of $\mathcal{X}$ is
distinguishable from any other members. Further, every member of $\mathcal{X}$ has no
intersection with any others. Lack of these assumptions is not convenient if 
one wants to know equality of two arguments for graphical and semantic constraints.\footnote{Alternatively, we can assume a domain of discourse 
and interpretation as customary in formal logic; see \cite{Kleene52} 
or any other standard text for the foundation of mathematical logic.} 
\begin{definition}[Arguments and argumentations] 
We define a (block) argument $a \in \mathcal{A}$ 
to be either $(\{x\}, \emptyset, \emptyset)$ for some $x \in 
\mathcal{X}$ 
or $(\{a_1, \ldots, a_n\}, R_p,$ $\Rs_p)$ 
for some $a_1, \ldots, a_n \in \mathcal{A}$ and 
some binary relations $R_p$ and $\Rs_p$ 
over $\{a_1, \ldots, a_n\}$.  
We say $a \in \mathcal{A}$ is
unitary iff 
    $a$ is some $(\{x\}, \emptyset, \emptyset)$. \\
\indent We define a Block (Bipolar) Argumentation (\BBA)
to be some argument $(A, R, \Rs)\linebreak \in \mathcal{A}$. 
We say that it is finite iff 
the number of occurrences of symbols is finite in 
$A$.\footnote{
As in set theory (see \cite{Jech2000}), this is not equivalent 
to finiteness 
in the number of members of $A$.} 
We denote 
the class of all finite {\BBA}s by $\F$, a subclass of $\mathcal{A}$,
 and refer to its member by $\Fe{}$ with or without 
a subscript.   

\end{definition} 
\begin{example}[$\BBA$ argumentation] Denote the argumentation \fbox{A} in Section 2 by $\Fe{}$. 
\vspace{-0.2cm}
  \begin{itemize} 
    \item[] $\Fe{} = (\{a_1, \ldots, a_8\}, \{(a_4, a_2), (a_2, a_3), (a_7, a_6)\},
    \{(a_1, a_2), (a_8, a_5)\})$.
    \item[] $a_i = (\{x_i\}, \emptyset, \emptyset)$ for $1 \leq i \leq 5$. \hspace{0.35cm} 
     $a_6 = (\{a_1, a_2\}, \emptyset, \{(a_1, a_2)\})$.
    \item[] $a_7 = (\{a_2, a_4\}, \{(a_4, a_2)\}, \emptyset)$. \hspace{0.5cm}
    $a_8 = (\{a_6, a_7\}, \{(a_7, a_6)\}, \emptyset)$.
  \end{itemize}

\end{example} 
%For an emphasis 
%that $m$ is there only to determine 
%the type of an argument, we write $(A, R, \Rs, m)$ 
%synonymously as ${(A, R, \Rs)\!:\!m}$. One may safely 
%regard $m$ as a label for, or a type of, $(A, R, \Rs)$. 
%\begin{definition}[Block (bipolar) argumentation] 
%\ind{ We define a Block (Bipolar) Argumentation (\BBA)
%to be some block argument $(A, R, \Rs) \in \mathcal{A}$. 
%We say that it is finite iff 
%the number of occurrences of symbols is finite in 
%$A$.\footnote{
%As in set theory (see \cite{Jech2000}), this is not equivalent 
%to finiteness 
%in the number of members of $A$.} 
%We denote 
%the class of all finite {\BBA}s by $\F$ and 
%refer to its member by $\Fe{}$ with or without 
%a subscript.   
%}
%\end{definition}  

\subsection{Representation of Argument(ation)s}
% For technical brevity, we will make use of two 
% representations of argument(ation)s. 
% \begin{definition}[Contextual representation]  
%     A context:\footnote{A very typical 
%   notation in structural proof theory \cite{351148} that 
%   we just borrow here.} $\Gamma(-)$, with or without a subscript 
%   after $\Gamma$, is defined to be 
%   recognised by the following grammar:\\ 
%   \indent  $\Gamma(-) := - \ | \ (\{\Gamma(-)\} 
%   \cup A, R, \Rs)$   \\ 
% We assume that some $a \in \mathcal{A}$ 
% replaces $-$ 
% in a context $\Gamma(-)$ as 
% $\Gamma(a)$. We denote the class 
% of all contexts by $\mathcal{C}$. For any 
% $\Fe{}$, any $\Gamma(-) \in \mathcal{C}$ and any 
% $a \in \mathcal{A}$, if $\Fe{} = \Gamma(a)$, we say 
% that $\Gamma(a)$ is a contextual representation of $\Fe{}$.
% \end{definition}    
% \begin{definition}[Context membership]
%     For any $\Fe{}$, for any $\Gamma(-) \in \mathcal{C}$ 
%     and for any $(A_p, R_p, \Rs_p), a_q \in \mathcal{A}$, 
%     if $\Fe{} = \Gamma((A_p, R_p, \Rs_p))$ $\andC$ 
%       $a_q \in A_p$, then we say   
%      $a_q$ occurs in the context $\Gamma(-)$. 
% \end{definition}  

% \begin{proposition} 
%     It is possible that, for a given $\Fe{}$ and 
%     for $a_p$ that occurs in $\Fe{}$, 
%   $a_p$ occurs in more than one context. 
% \end{proposition}  

To refer to arguments in a specific position in $\Fe{}$, we make 
use of: 
\begin{definition}[Flat representation] 
  Let $\langle\mathbb{N}\rangle$ denote the class of all sequences 
  of natural numbers (an empty sequence is included), whose member is referred to by 
  $\langle m \rangle$ with or without a subscript, 
  or by a specific sequence of natural numbers. We use `.'
  for sequence concatenation. 
  Let $\varpi: \mathcal{A}  \rightarrow 2^{\mathcal{A} \times 
   \langle \mathbb{N} \rangle}$ be such that 
  $\varpi(a)$ is the least set   
  that satisfies all the following. 
   \begin{enumerate}    
      \item $(a, 0) \in \varpi(a)$. 
      \item For any $(a_p, \langle m \rangle) \in 
       \varpi(a)$, 
        if $a_p$ is not unitary and is some $\{a_1, \ldots, a_{n_p}\}$, then 
        $(a_i, \langle m \rangle.n_i) \in \varpi(a)$ 
        for every $1 \leq i \leq {n_p}$ such that all 
        $n_1, \ldots, n_{n_p}$ are distinct.
   \end{enumerate}  
   For any $a$ and any $n$, we say that $\varpi(a)$ 
   is its flat representation. 
\end{definition}

\begin{example}[Flat representation]\label{ex_2} ({\it Continued}) 
For the same example argumentation $\Fe{}$ in \fbox{A}, we can 
   define its flat representation $\varpi(\Fe{})$ 
   to be the set of all the following. Just for 
   disambiguation,  we demarcate the constituents in a sequence of natural numbers with `.'. 
  \begin{itemize} 
     \item[] $(\Fe{}, 0)$. \hspace{0.6cm} 
      $(a_{i}, 0.{i})$ for $1 \leq i \leq 8$. \hspace{0.6cm} 
      $(a_{1}, 0.6.{1})$ and $(a_{2}, 0.6.{2})$.
     \item[] $(a_{2}, 0.7.{1})$ and $(a_{4}, 0.7.2)$. \hspace{3.1cm} 
      $(a_{6}, 0.8.{1})$ and $(a_{7}, 0.8.{2})$.
     \item[] $(a_{1}, 0.8.1.1)$ and $(a_{2}, 0.8.1.2)$. \hspace{2.2cm}
      $(a_2, 0.8.2.1)$ and $(a_4, 0.8.2.2)$.
  \end{itemize}  
\end{example}

\begin{definition}[Order in flat representation] 
    Let $\sqsupseteq: ({\mathcal{A}} \times \langle \mathbb{N}
   \rangle) \times ({\mathcal{A}} \times \langle \mathbb{N}
   \rangle)$ be such that $((a_p, \langle m_p \rangle),
   (a_q, \langle m_q \rangle)) \in\ \sqsupseteq$,
   synonymously written as $(a_p, \langle m_p \rangle) 
   \sqsupseteq (a_q, \langle m_q \rangle)$, iff  
   $\langle m_p \rangle.\langle m_r \rangle
 = \langle m_q \rangle$ for 
   some $\langle m_r \rangle$. We write 
   $(a_p, \langle m_p \rangle) \sqsupset (a_q, \langle m_q 
    \rangle)$ iff 
   $(a_p, \langle m_p \rangle) \sqsupseteq (a_q, \langle 
      m_q \rangle)$ $\andC$ $(a_q, \langle m_q \rangle) \not\sqsupseteq (a_p, \langle 
      m_p \rangle)$.\footnote{``$\andC$" instead of ``and" is used in this paper when the context in which 
      it appears strongly indicates truth-value comparisons. It follows the semantics 
      of classical logic conjunction.}
\end{definition} 
% While we can clearly derive the other representation from one 
% representation for a given block argumentation, contextual 
% representation handles modification to a part of block argumentation
% more concisely than the flat representation, whereas 
% the latter allows 
% an easier reference to it. \ryuta{Is it 
% possible to merge the two representations?}
\subsection{Characterisation of Complete Sets with No Constraints}
We characterise complete sets with no constraints initially. 

\begin{definition}[Arguments, attacks and supports in $\langle m \rangle$]  
Let $\textsf{Arg}: \F \times \langle \mathbb{N} \rangle \rightarrow 
2^{\mathcal{A}}$ be such that 
$\textsf{Arg}(\Fe{}, \langle m \rangle) = A_p$ iff 
$((A_p, R, \Rs), \langle m \rangle) \in \varpi(\Fe{})$ for some 
$R$ and $\Rs$.  Let $\getAttack, \getSupport:
%\getActualAttack, 
%\getActualSupport: 
\F \times \langle \mathbb{N} \rangle 
\rightarrow 2^{\mathcal{A} \times \mathcal{A}}$ 
be such that:
\vspace{-0.2cm} 
\begin{itemize}
    \item[]  {\small$\getAttack(\Fe{}, \langle m \rangle) = 
        \{(a_1, a_2) \in R \ | \ \exists ((A, R, \Rs), \langle m \rangle) \in \varpi(\Fe{}).\ a_1, a_2 \in A\}$}
    \item[]  {\small$\getSupport(\Fe{}, \langle m \rangle) = \{(a_1, a_2) \in \Rs \ | \ 
            \exists ((A, R, \Rs), \langle m \rangle) \in \varpi(\Fe{}).\ a_1,
             a_2 \in A\}$}
\end{itemize}   
\vspace{-0.2cm} 
\noindent For any $\Fe{}$ and any $\langle m \rangle$, 
we say: 
\vspace{-0.2cm} 
\begin{itemize}
    \item[] $A_p$ is the set of arguments in $\langle m \rangle$ 
        iff $A_p = \textsf{Arg}(\Fe{}, \langle m \rangle)$.
	\item[]  $a_1$ attacks $a_2$ in $\langle m \rangle$ 
             iff $(a_1, a_2) \in \getAttack(\Fe{}, \langle m \rangle)$. 
	\item[] $a_1$ supports $a_2$ in $\langle m \rangle$ 
    	      iff $(a_1, a_2) \in \getSupport(\Fe{}, \langle m \rangle)$.  
\end{itemize} 
%If 
%${(a_1, a_2)\!:\!\Gamma(-)} \in \getActualAttack(A^F)$ 
%or ${(a_1, a_2)\!:\!\Gamma(-)} \in \getActualSupport(A^F)$, 
%then, in particular, $a_1, a_2 \in A_1$ 
%in some \textsf{act} argument ${(A_1, R_1, \Rs_1, m_1)\!:\!\textsf{act}}$.}
\end{definition}   
While there are three typical interpretations (deductive, necessary, evidential)
of support in the literature, they enforce a strong dependency 
between arguments and the arguments that support them concerning 
their acceptance. In light of the example in Section~\ref{sec:motivation}, 
our interpretation of support here is weaker, almost supplementary, as 
in the following definitions. Briefly, it is not necessary that an argument be in 
a complete set when its supporter/supportee is in the set. A supporter can, however, prevent an argument attacked by an attacker from being strongly 
rejected (with labels, it concerns the difference of whether the argument gets $-$ (which leads to strong 
rejection) or ?). 
We look at extension-based complete set characterisation first. 
\begin{definition}[Extension-based complete set when no constraints]
   For any $\Fe{}$ and any $\langle m \rangle$, we say:  
   $A_1$ defends $a$ in $\langle m \rangle$ iff 
   $A_1 \subseteq \textsf{Arg}(\Fe{}, \langle m \rangle)$ $\andC$
   $a \in \textsf{Arg}(\Fe{}, \langle m \rangle)$ $\andC$ 
   every $a_1$ attacking $a$ in $\langle m \rangle$ is: 
  attacked by at least some $a_2 \in A_1$; and  
    not supported by any $a_3 \in A_1$. \\
    \indent We say that 
    $A_1$ is conflict-free in $\langle m \rangle$ iff 
    $A_1 \subseteq \textsf{Arg}(\Fe{}, \langle m \rangle)$ $\andC$
    $(a_1, a_2) \not\in \getAttack(\Fe{}, \langle m \rangle)$
    for any $a_1, a_2 \in A_1$. \\
    \indent We say that $A_1$ is standard complete in $\langle m \rangle$ iff 
    $A_1 \subseteq \textsf{Arg}(\Fe{}, \langle m \rangle)$ $\andC$
    $A_1$ is conflict-free and 
    includes all arguments it defends in $\langle m \rangle$.
\end{definition}
\begin{example}[Complete set with no constraints]{\it (Continued)} For $\Fe{}$ in \fbox{A}, we have 
only one standard complete set in 0 with this characterisation, namely 
$\{a_1, a_4, a_5, a_7, a_8\}$. To explain the role of a supporter to prevent 
strong rejection of an argument, notice that $a_1$, which is in the standard complete set 
in 0, supports $a_2$, which is attacked by $a_4$ in the same standard complete set in 0. 
If it were not for the supporter, $a_3$ would be in the standard complete set in 0.
\end{example}
We can also have a label-based characterisation with  
$\mathcal{L} (= \{+,-,?\})$.
%\begin{definition}[Labelling]  
    % Let $\Lambda$ be the class of 
    % all functions $\lambda: 2^{\mathcal{A} \times 
    %   \langle \mathbb{N} \rangle} \rightarrow 
    %     2^{\mathcal{L} \times  \langle \mathbb{N} 
    %          \rangle}$ each of which is such that 
    % $\lambda(X) = \{(l, \langle m \rangle) \ | \ l \in \mathcal{L}
    %  \ \andC\ \exists 
    %      (a, \langle m \rangle) \in X \ \andC\ 
    %      \forall (l_1, \langle m_1 \rangle) \in X.
    %       \langle m_1 \rangle = \langle m \rangle 
    %       \text{ materially implies } 
    %       l = l_1\}$. 
    % For any $\Fe{}$, for any $n$ 
    % and for any $\lambda \in \Lambda$, we say that $\lambda(\varpi(\Fe{}, n))$ is 
    % $\varpi(\Fe{}, n)$'s label.  \\
    % \indent We define $\chi: 2^{\mathcal{L} \times \langle \mathbb{N} \rangle} \times \langle \mathbb{N} \rangle \rightarrow 
    % \mathcal{L}$ to be such that 
    % $\chi(X, \langle m \rangle) = l$ iff
    % $(l, \langle m \rangle) \in X$. 
%\end{definition} 
% \begin{proposition}
%   For any flat representation $\varpi(\Fe{}, n)$ of $\Fe{}$ and for each 
%   $\lambda \in \Lambda$, it holds that 
%   $\varpi(\Fe{}, n) = \lambda^{-1}(\lambda(\Fe{}, n))$, i.e. 
%   $\lambda^{-1} \circ \lambda$ is an identity function.
% \end{proposition}
\begin{definition}[Complete labelling when no constraints] 

 Let $\Lambda$ be the class of 
    all  $\lambda: \mathcal{A} \times 
      \langle \mathbb{N} \rangle \rightarrow 
        \mathcal{L}$
             such that 
    $\lambda((a, \langle m \rangle)) = l$ for some 
    $l$. 
   % Let $\oplus, \ominus, \bigcirc: \F \times \langle \mathbb{N} \rangle \times 
    % \mathbb{N}$ 
    % be such that:
    % \begin{itemize}
    %     \item $\oplus(\Fe{}, \langle m \rangle, i)$ iff 
    %       (1) there exists some $(a_p, \langle m \rangle.i) \in \varpi(\Fe{}, 0)$ 
    %       $\andC$ (2) for every $j$ and for every $(a_q, \langle m \rangle.j)
    %       \in \varpi(\Fe{}, 0)$, if $a_q$ attacks $a_p$ in $\langle m \rangle$, 
    %       then $\ominus(\Fe{}, \langle m \rangle.j)$.
    %     \item $\ominus(\Fe{}, \langle m \rangle, i)$ iff (1) there exists some 
    %     $(a_p, \langle m \rangle.i) \in \varpi(\Fe{}, 0)$ $\andC$ 
    %   (2)  there exists some $(a_q, \langle m \rangle.j) \in \varpi(\Fe{}, 0)$ 
    %     such that $\oplus(\Fe{}, \langle m \rangle, j)$ and that 
    %     $a_q$ attacks $a_p$ in $\langle m \rangle$ $\andC$ (3) 
    %     there is no $(a_r, \langle m \rangle.k) \in \varpi(\Fe{}, 0)$ 
    %     such that $\oplus(\Fe{}, \langle m \rangle, ,)$ and that 
    %     $a_r$ supports $a_p$ in $\langle m \rangle$.
    % \end{itemize}
   For any $\Fe{}$ and any $\lambda \in \Lambda$, we say that 
   $\lambda$ is a standard complete labelling of $\Fe{}$ iff
   every $(\{a_1, \ldots, a_n\}, \langle m \rangle) \in \varpi(\Fe{})$ 
   satisfies
    all the following. 
   \begin{itemize}[leftmargin=0.3cm]
       \item $\lambda((a_i, \langle m \rangle.n_i)) = +$, 
          $1 \leq i \leq n$, 
           iff  every $a_j$, $1 \leq j \leq n$, attacking $a_i$ in $\langle m \rangle$
          satisfies $\lambda((a_j, \langle m \rangle.n_j)) = -$.
        \item $\lambda((a_i, \langle m \rangle.n_i)) = -$, 
          $1 \leq i \leq n$,  iff 
               there exists some $1 \leq j \leq n$ such that 
               $\lambda((a_j, \langle m \rangle.n_j))\\ = +$ 
               and that $a_j$ attacks $a_i$ in $\langle m \rangle$
               $\andC$ there is no $1 \leq k \leq n$ such that 
               $\lambda((a_k, \langle m \rangle.n_k)) = +$ 
               and that $a_k$ supports $a_i$ in $\langle m \rangle$.
        % \item $\lambda((a_i, \langle m \rangle.i)) =\ ?$,
        % $1 \leq i \leq n$,  
        % if 
        %       there exist some $1 \leq j, k \leq n$ such that 
        %       $\lambda((a_j, \langle m \rangle.j)) = \lambda((a_k, \langle m \rangle.k)) = +$, 
        %       that $a_j$ attacks $a_i$ in $\langle m \rangle$, and that 
        %       $a_k$ supports $a_i$ in $\langle m \rangle$. 
   \end{itemize}
\end{definition}
\begin{theorem}[Correspondence between standard complete sets and standard complete labellings]
   For any $\Fe{}$ and any $\langle m \rangle$, 
   $A_1 \subseteq \textsf{Arg}(\Fe{}, \langle m \rangle)$ is standard complete in $\langle m \rangle$
   only if there is some standard complete labelling $\lambda$ of $\Fe{}$ 
   such that
   $\lambda((a_p, \langle m \rangle.n)) = +$  is equivalent to 
   $a_p \in A_1$ for any $(a_p, \langle m \rangle.n) \in  \varpi(\Fe{})$. Conversely, $\lambda$ is a standard complete labelling of $\Fe{}$ 
   only if, for every $(a, \langle m \rangle) \in \varpi(\Fe{})$,
  {\small $\{a_p \in \textsf{Arg}(\Fe{}, \langle m \rangle) \ | \linebreak
  \exists n \in \mathbb{N}.\lambda((a_p, \langle m \rangle.n)) = +\}$} is a standard complete set in $\langle m \rangle$.  
\end{theorem}
% \begin{proof} 
%   Assume $A_1 \subseteq \textsf{Arg}(\Fe{}, 
%   \langle m \rangle)$ is standard 
%   complete in $\langle m \rangle$. 
%   Suppose some $\lambda \in \Lambda$ such that 
%   $\lambda((a_p, \langle m \rangle.n)) = 
%   +$ iff $a_p \in A_1$ for any  
%      $(a_p, \langle m \rangle.n) \in \varpi(\Fe{},
%      0)$. Suppose, by way of showing 
%      contradiction, that there necessarily
%      exists some $a_q \not\in A_1$ 
%      such that 
%      $\lambda((a_q, \langle m \rangle.{n_q})) =
%      +$. Assume $A_2 \subseteq \textsf{Arg}(
%      \Fe{}, \langle m \rangle)$ is the set 
%      of all arguments in $\langle m \rangle$ 
%      attacking $a_q$. 
% \end{proof} 
On the basis of this correspondence, we will work mainly with labels, as they 
simplify referrals of arguments that do not get $+$. 

\subsection{Graphical (Syntactic) Constraints} 
\vspace{-0.1cm}
We now characterise constraints, which can be graphical (syntactic) or 
semantic. The former enforces that any argumentation within an argument(ation) has already occurred, while the latter enforces 
       that 
    acceptability statuses of  
    arguments in any argumentation respect in a certain way those 
    of the same arguments that have already occurred. 
We begin with a graphical one, for which 
we need to be able to tell equality of two arguments.
\begin{definition}[Argument equality] 
  Let $\equalTo: \mathcal{A} \times \mathcal{A}$ 
  be such that $\equalTo(a_1, a_2)$ iff one of the 
     following conditions is satisfied. 
     \begin{itemize}[leftmargin=0.3cm]
         \item $a_1 = a_2 = (\{x\}, \emptyset, \emptyset)$ 
            for some $x \in \mathcal{X}$.
         \item If $a_1$ is some $(\{a''_{1}, \ldots, a''_{n}\}, R_1, \Rs_1)$,
            then $a_2$ is some $(\{a'_{1}, \ldots, a'_{n}\}, R_2, \Rs_2)$
            such that: 
              \begin{itemize}
                  \item 
                    $\equalTo(a''_{i}, a'_{i})$;
                    \item 
               $(a''_{i}, a''_{j}) \in R_1$ iff 
               $(a'_{i}, a'_{j}) \in R_2$;
               \item and $(a''_{i}, a''_{j}) \in \Rs_1$ iff 
               $(a'_{i}, a'_{j}) \in \Rs_2$ 
           \end{itemize}
           for $1 \leq i, j \leq n$.
     \end{itemize}
\end{definition} 
\begin{definition}[Sub-argumentation]{\ }
    Let $\searrow: \mathcal{A} \times \mathcal{A}$ be such that
$((A_1, R_1, \Rs_1), (A_2,$ $R_2,\Rs_2)) \in \searrow$, 
written synonymously as $(A_1, R_1, \Rs_1) \searrow 
(A_2, R_2, \Rs_2)$ iff $A_2 = \{a'_{1}, \ldots,$  $a'_{n_2}\}$ 
and $A_1 = \{a_{1}, \ldots, a_{n_1}\}$ are such that: 
\vspace{-0.1cm}
\begin{itemize}[leftmargin=0.3cm]
    \item $|A_2| \leq |A_1|$. 
    \item For each $a'_i$, $1 \leq i \leq n_2$, there exists 
        some $a_j$, $1 \leq j \leq n_1$, such  that 
        $\equalTo(a'_i, a_j)$.
    \item For each two $a'_{i_1}, a'_{i_2}$, $1 \leq i_1, i_2 \leq n_2$, if 
     $(a'_{i_1}, a'_{i_2}) \in R_2$, then there exist
     some $a_{j_1}, a_{j_2}$, $1 \leq j_1, j_2 \leq n_1$ such that 
     $(a_{j_1}, a_{j_2}) \in R_1$, that 
     $\equalTo(a'_{i_1}, a_{j_1})$, and that 
     $\equalTo(a'_{i_2}, a_{j_2})$.
     \item For each two $a'_{i_1}, a'_{i_2}$, $1 \leq i_1, i_2 \leq n_2$, if 
     $(a'_{i_1}, a'_{i_2}) \in \Rs_2$, then there exist
     some $a_{j_1}, a_{j_2}$, $1 \leq j_1, j_2 \leq n_1$ such that 
     $(a_{j_1}, a_{j_2}) \in \Rs_1$, that 
     $\equalTo(a'_{i_1}, a_{j_1})$, and that 
     $\equalTo(a'_{i_2}, a_{j_2})$.
\end{itemize}
 We say that $a_2$ is a sub-argumentation of 
 $a_1$ iff $a_1 \searrow a_2$. 
\end{definition}
\begin{definition}[Graphical (syntactic) constraints] 
For any $\Fe{}$,
  we say that\linebreak $(a_2, \langle m_2 \rangle) \in \varpi(\Fe{})$ satisfies 
%   \begin{description} 
%      \item[Gu] iff, 
%         if $a_2$ is unitary $\andC$ $\langle m_2 \rangle \not= 0$, then there exists 
%          some $(a_1, \langle m_1 \rangle) \in \varpi(\Fe{}, 0)$ such that 
%          $(a_1, \langle m_1 \rangle) \sqsupset (a_2, \langle m_2 \rangle)$ and 
%          that $a_1 \searrow a_2$.
      \textbf{G} iff both of the conditions below hold.
      \vspace{-0.1cm}
      \begin{enumerate}
          \item If $\langle m_2 \rangle \not= 0$ and if 
            $a_2$ is not unitary, then 
              there exists some $(a_1, \langle m_1 \rangle) \in 
      \varpi(\Fe{})$ such that $(a_1, \langle m_1 \rangle) \sqsupset
      (a_2, \langle m_2 \rangle)$ and that $a_1 \searrow a_2$.
           \item Every $(a_p, \langle m_2\rangle.n) \in \varpi(\Fe{})$
            satisfies \textbf{G}. 
          \end{enumerate}
%   \end{description}
 \end{definition}   
\begin{example}[Graphical (syntactic) constraints] {\it (Continued)}
%   \begin{center} 
%       \includegraphics[scale=0.14]{jelia19-1/assassinationFigure500.pdf}
%       \includegraphics[scale=0.14]{jelia19-1/example5000.pdf}
%       \includegraphics[scale=0.14]{jelia19-1/assassinationFigure1000.pdf}
%       \includegraphics[scale=0.14]{jelia19-1/assassinationFigure2000.pdf}
%   \end{center} 
   As has been the case so far, let $\Fe{}$ be the argumentation in \fbox{A}. 
    and let its flat representation be as given in Example \ref{ex_2}. 
   Then $\Fe{}$ clearly satisfies $\textbf{G}$ for any 
   $(a, \langle m \rangle) \in \varpi(\Fe{})$, since: 
   \begin{itemize}
       \item $(\Fe{}, 0)$: there is nothing to show, as the sequence is 0. 
       \item $(a_i, 0.i)$, $1 \leq i \leq 5$: there is nothing to show, as 
                            $a_{1,\ldots,5}$ are unitary.
       \item $(a_6, 0.6)$: $\Fe{} \searrow (\{a_1, a_2\}, \emptyset, \{(a_1, a_2)\})$.
       \item $(a_i, 0.6.i)$, $1 \leq i \leq 2$: there is nothing to show, 
          as $a_1$ and $a_2$ are unitary.
       \item $(a_7, 0.7)$: $\Fe{} \searrow (\{a_2, a_4\}, \{(a_4, a_2)\}, \emptyset)$.
       \item $(a_2, 0.7.1), (a_4, 0.7.2)$: there is nothing to show, 
       as $a_2$ and $a_4$ are unitary. 
       \item $(a_8, 0.8)$: $\Fe{} \searrow (\{a_6, a_7\}, \{(a_7, a_6)\}, \emptyset)$.
       \item $(a_6, 0.8.1)$: see above for $(a_6, 0.6)$.
       \item $(a_7, 0.8.2)$: see above for $(a_7, 0.7)$.
       \item $(a_i, 0.8.1.i), 1 \leq i \leq 2$: there is nothing to show. 
       \item $(a_2, 0.8.2.1)$ and $(a_4, 0.8.2.2)$: there is nothing to show.
   \end{itemize}
   For comparison, however, suppose that $\Fe{}$ is the argumentation 
   in \fbox{B}, i.e. $\Fe{} = (\{a_6, a_7\}, \{a_7, a_6\}, \emptyset)$, 
   $a_6 = (\{a_1, a_2\}, \emptyset, \{(a_1, a_2)\})$, and 
   $a_7 = (\{a_2, a_4\}, \{(a_4, a_2)\}, \emptyset)$. Assume 
   $(a_6, 0.1), (a_7, 0.2) \in \varpi(\Fe{})$, then neither 
   of them satisfies \textbf{G}, because none of $a_1, a_2, a_4$ are 
   in $\textsf{Arg}(\Fe{}, 0)$. 
\end{example} 
Since the graph structure of a given $\Fe{}$ never changes, 
violation of the graphical constraint monotonically propagates up to $\langle m \rangle = 0$ (0 excluded) from longer sequences. Thus, it is rather straightforward to handle graphical 
constraint satisfaction. 
\subsection{Semantic constraints}  
By contrast, semantic constraint satisfaction depends on what labels are assigned 
to arguments, which adds to technical subtlety. We define a partial order 
on $\mathcal{L}$, and characterise semantic constraints based on them. 
\begin{definition}[Order in labels]
    Let $\succeq: \mathcal{L} \times \mathcal{L}$ 
    be $\{(?, +), (?, -), (+, +), (-, -), (?, ?)\}$. 
    We write $(l_1, l_2) \in \succeq$ alternatively
    as $l_1 \succeq l_2$.
\end{definition}
\begin{definition}[Semantic constraints]
    For any $\Fe{}$ and for any $\lambda \in \Lambda$, 
     we say that $(a_2, \langle m_2 \rangle) \in \varpi(\Fe{})$ and $\lambda$ satisfy:
    \begin{description}
      \item[S] iff for every $(a_1, \langle m_1 \rangle)
       \in \varpi(\Fe{})$, 
        if $\equalTo(a_1, a_2)$ $\andC$ $(a_1, \langle m_1 \rangle)  \sqsupset (a_2, \langle m_2 \rangle)$, 
        then $\lambda((a_1, \langle m_1 \rangle)) \succeq \lambda((a_2, \langle m_2 \rangle))$.
    %   \item[S$\exists$] iff,
    %     if
    %     $(\Fe{}, 0) \sqsupset (a_2, \langle m_2 \rangle)$, then there exists
    %     some $(a_1, \langle m_1 \rangle) \in \varpi(\Fe{}, 0)$
    %     such that  $(a_1, \langle m_1 \rangle)  \sqsupset 
    %     (a_2, \langle m_2 \rangle)$ and also that
    %     $\lambda((a_1, \langle m_1 \rangle)) \succeq \lambda((a_2, \langle m_2 \rangle))$.
       \item[$\star$] 
    %iff, for any $(a_3, \ang{m_3}), (a_4, \ang{m_4}) \in \varpi(\Fe{}, 0)$, 
    %  if $\ang{m_3} = \ang{m_2}.n_3$ $\andC$ $\ang{m_4} = \ang{m_2}.n_4$ $\andC$ 
    %  $\equalTo(a_3, a_4)$, then $\lambda((a_3, \ang{m_3})) = \lambda((a_4, \ang{m_4}))$.

     iff, for any 
     $(a_1, \langle m_1 \rangle) \in 
     \varpi(\Fe{})$, if there exist some $\langle m_3\rangle$, $n_1$ and 
     $n_2$ such that $\langle m_3 \rangle.n_1 = \langle m_1 \rangle$ and that 
     $\langle m_3 \rangle.n_2 = \langle m_2 \rangle$, and if 
     $\equalTo(a_1, a_2)$, then $\lambda((a_1, \langle m_1 \rangle))\linebreak = \lambda((a_2, \langle m_2 \rangle))$. 
    \end{description}
\end{definition}
     To speak of the use of $\succeq$, there should be no oddity from a semantic
     coherency perspective when $\lambda((a_1, \langle m_1)) = 
     \lambda((a_2, \langle m_2 \rangle))$
     for every $(a_1, \langle m_1 \rangle), (a_2, \langle m_2 \rangle) 
     \in \varpi(\Fe{})$ such that $\equalTo(a_1, a_2)$. That just means that same arguments 
     in $\Fe{}$ are assigned the same label. However, 
     consider our example \fbox{A}. There, $a_2$ in $\textsf{Arg}(\Fe{}, 0)$ which is assigned $?$ by a standard complete labelling is given 
     $-$ in $\textsf{Arg}(a_6, 0.6)$ by the Malaysian authorities and $+$ in $\textsf{Arg}(a_7, 0.7)$ by the defence lawyer. $?$ assigned to 
     an argument in $\langle m \rangle$, therefore, 
     can be interpreted flexibly in argumentation in $\langle m \rangle.\langle 
     m_p \rangle$, that it can be any of $+$, $-$, $?$ depending 
     on which part of the argumentation in $\langle m \rangle$ is selected to 
     be included in the argumentation in $\langle m \rangle.\langle m_p \rangle$
     (for some non-empty $\langle m_p \rangle$).
     This is the intuition for $\succeq$ and its use in $\textbf{S}$. \\
     \indent The $\star$ symbol denotes a semantic constraint among the members of 
     $\textsf{Arg}(\Fe{}, \langle m \rangle)$ for some $\langle m \rangle$, to prevent the same arguments 
     from being assigned a different label.%\todo{Maybe more space between arguments in the image below? So that ``Attacks'' does not overlap the grey of nodes}
     \begin{example}[Semantic constraints]\label{ex_3}
   \begin{center}
       \includegraphics[scale=0.11]{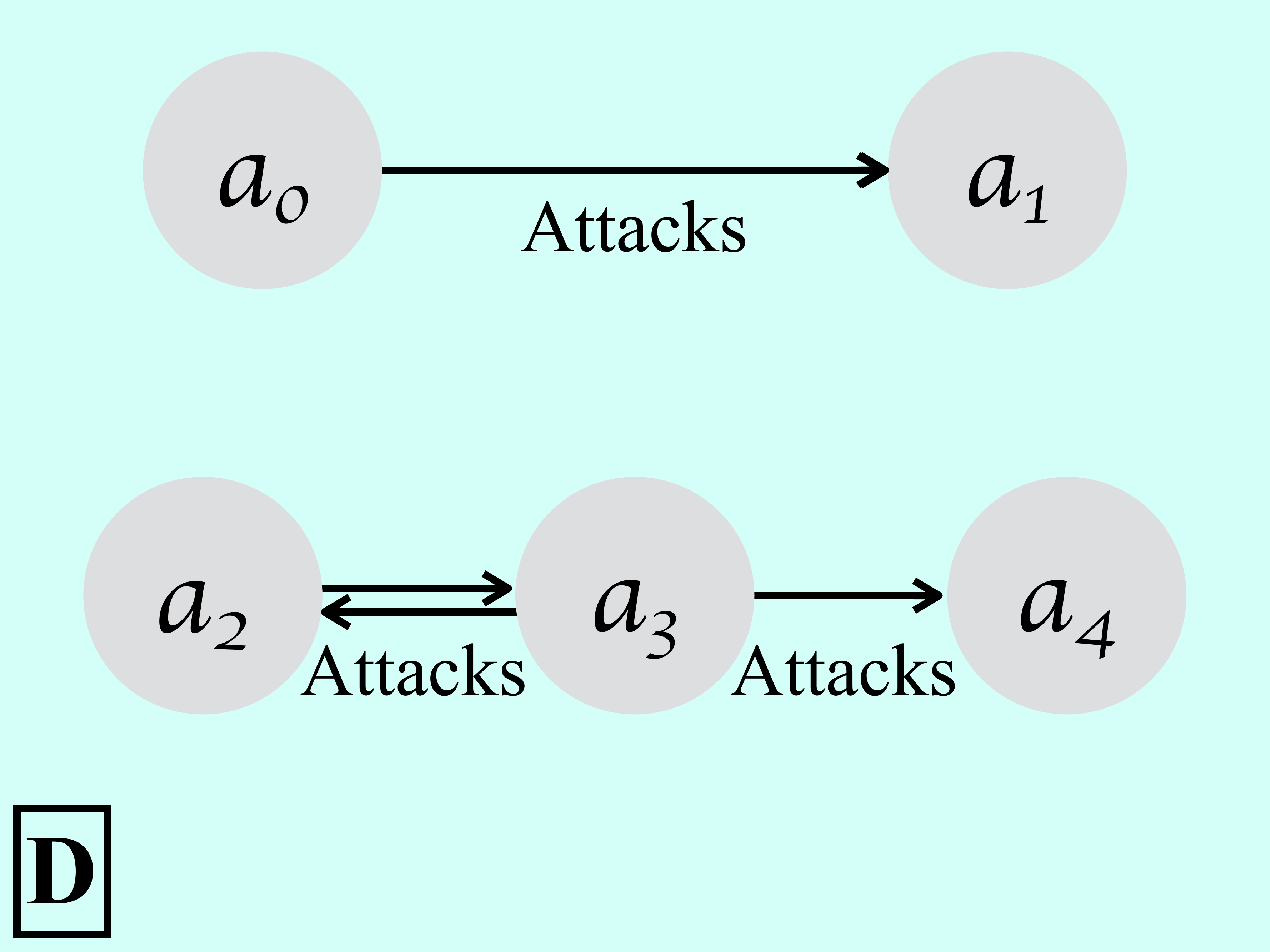}
       \includegraphics[scale=0.11]{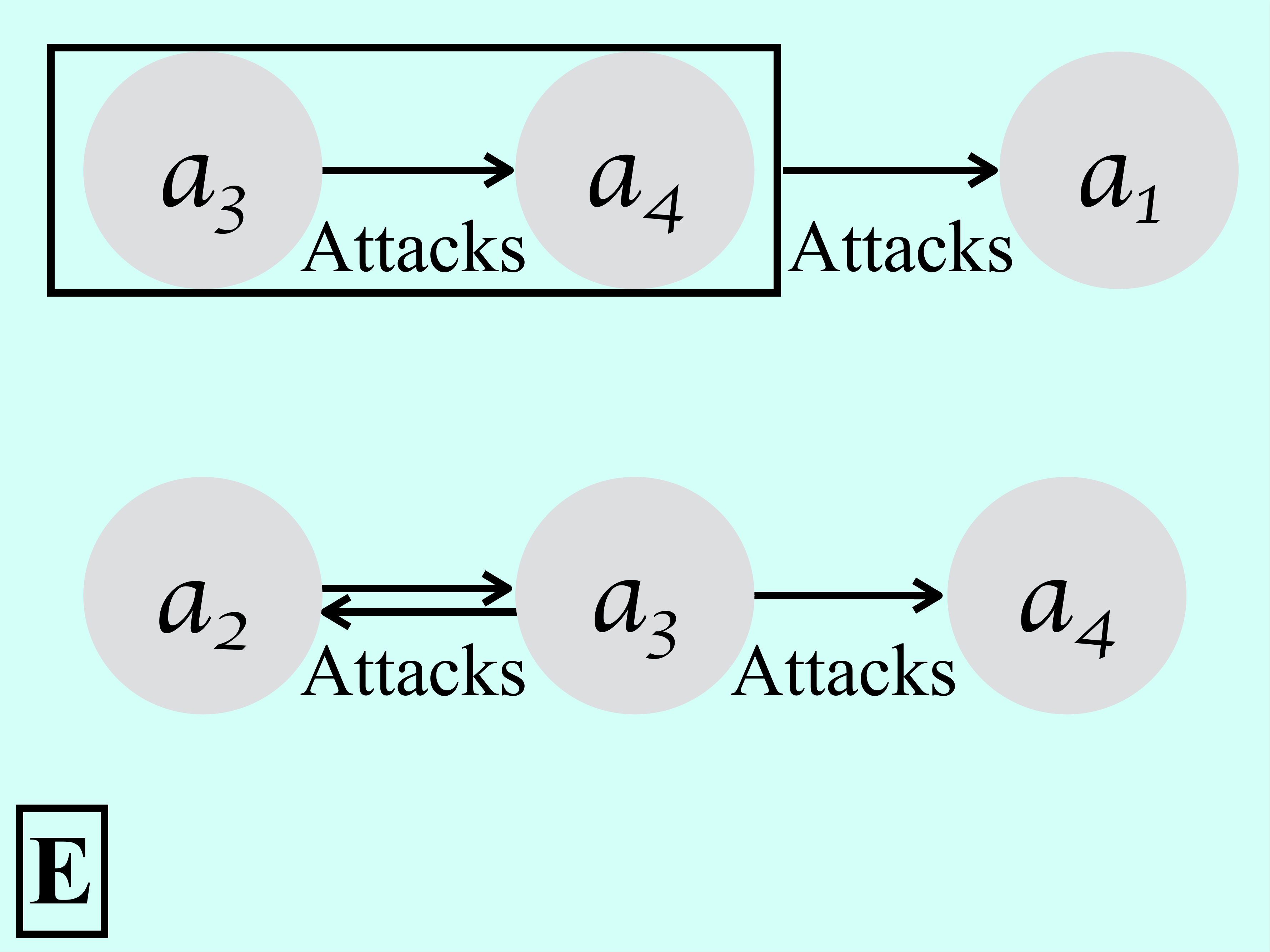}
       \includegraphics[scale=0.11]{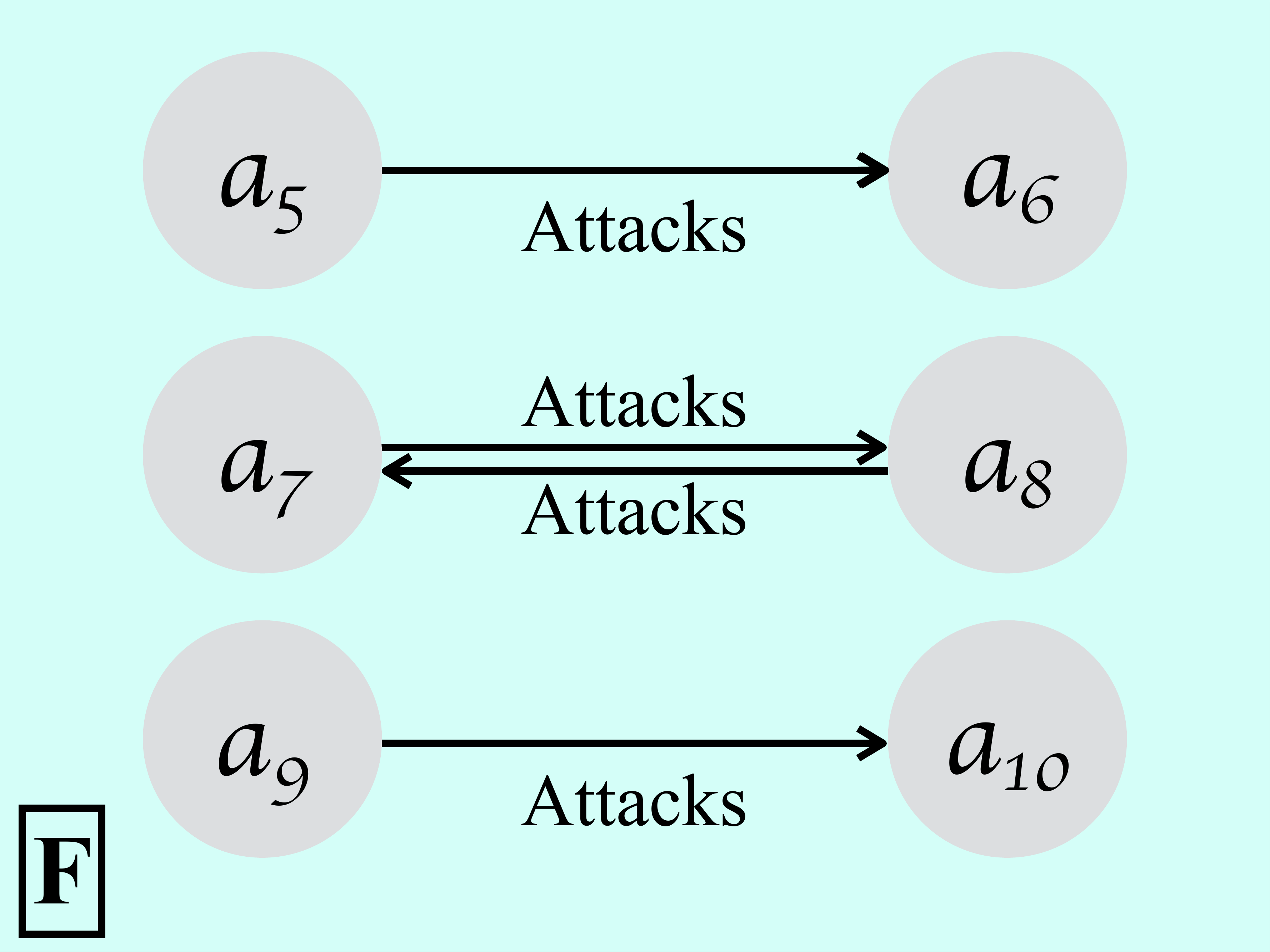}
   \end{center}
   Denote the argumentation in \fbox{D} by $\Fe{}$. Assume 
   that $a_0 = (\{a_3, a_4\}, \{(a_3, a_4)\}, \emptyset)$ as 
   shown in \fbox{E}, and assume that 
   $a_1, \ldots, a_4$ are all unitary such that 
   $\equalTo(a_i, a_j)$ for no distinct $1 \leq i, j \leq 4$.
   Assume $(\Fe{}, 0), (a_0, 0.0), (a_1, 0.1), (a_2, 0.2), (a_3, 0.3),
   (a_4, 0.4),  \linebreak(a_3, 0.0.1), (a_4, 0.0.2) \in \varpi(\Fe{})$. 
   The following are all the standard complete labelling distinct for 
   $\Fe{}$. 
   \begin{itemize}
       \item $\lambda_1((a_i, 0.i)) = +$ for 
       $i \in \{0, 2, 4\}$, $\lambda_1((a_i, 0.i)) = -$ for $i \in \{1, 3\}$, \\
       $\lambda_1((a_3, 0.0.1)) = +$, $\lambda_1((a_4, 0.0.2)) = -$. 
       \item $\lambda_2((a_i, 0.i)) = +$ for 
       $i \in \{0, 3\}$, $\lambda_2((a_i, 0.i)) = -$ for $i \in \{1, 2, 4\}$.\\
        $\lambda_2((a_3, 0.0.1)) = +$, $\lambda_2((a_4, 0.0.2)) = -$.
       \item $\lambda_3((a_0, 0.0)) = +$, $\lambda_3((a_1, 0.1)) = -$, 
       $\lambda_3((a_i, 0.i)) =\ ?$ for $i \in \{2, 3, 4\}$.\\
        $\lambda_3((a_3, 0.0.1)) = +$, $\lambda_3((a_4, 0.0.2)) = -$.
   \end{itemize}
   $\lambda_1$ and $((a_3, 0.0.1))$ do not satisfy \textbf{S}. $\lambda_2$ and $((a_3, 0.0.1))$ do. Now, while $\lambda_3((a_3, 0.3)) = ? \not= + = 
   \lambda_3((a_3, 0.0.1))$ and $\lambda_3((a_4, 0.4)) =\ ? \not= - = 
   \lambda_3((a_4, 0.0.2))$, \textbf{S} is respected, since 
  $? \succeq +, -$. \\
   \indent Denote the argumentation in \fbox{F} by $\Fe{}$. Assume 
   $(\Fe{}, 0), (a_i, 0.i) \in \varpi(\Fe{})$ for $5 \leq i \leq 10$. Assume
   $\equalTo(a_5, a_7)$ and  $\equalTo(a_9, a_8)$. There is a standard complete labelling $\lambda$ of $\Fe{}$ 
   such that $\lambda((a_5, 0.5)) = \lambda((0.7, 0.7)) = \lambda((a_9,0.9)) = +$, 
   $\lambda((a_6, 0.6)) = \lambda((a_8, 0.8)) = \lambda((a_{10}, 0.10)) = -$. 
   Then it could be understood that the same 
   argument ($a_8$ and $a_9$) in 0 is assigned both + and $-$ simultaneously, which on some occasions 
   outside of rhetoric argumentation 
   is not convenient. $\star$ is the condition against such labelling. 
\end{example}
% \begin{definition}[Violation]
%   Let $\mathcal{C}$ be $\{\textbf{Gu}, \textbf{G}\}$. Let $\violate: \F \times \langle \mathbb{N} \rangle \times 2^\mathcal{C}$ be such that 
%   $\violate(\Fe{}, \langle m \rangle, c)$ for some $c \subseteq \mathcal{C} $ iff all the following 
%   conditions satisfy.
%   \begin{itemize}
%       \item $c \not= \emptyset$. 
%       \item If $\textbf{G} \in c$, then any $(a, \langle m \rangle) \in \varpi(\Fe{},0)$ satisfies
%             $\textbf{G}$.
%       \item If $\textbf{Gu} \in c$, then any $(a, \langle m \rangle) \in \varpi(\Fe{},0)$ satisfies
%             $\textbf{Gu}$.
%     %     \item If $s = \textbf{S}\forall$, then $\lambda$ and any $(a, \langle m \rangle) \in \varpi(\Fe{})$ satisfy
%     %         $\textbf{S}\forall$.
%     %   \item If $s = \textbf{S}\exists$, then $\lambda$ and any $(a, \langle m \rangle) \in \varpi(\Fe{})$ satisfy
%     %         $\textbf{S}\exists$.
%     %     \item If $s = \star$, then $\lambda$ and any 
%     %     $(a, \langle m \rangle) \in \varpi(\Fe{})$ satisfy
%     %         $\star$.
%   \end{itemize}
%   For any $\Fe{}$, any $\langle m \rangle$ and any $c \subseteq \mathcal{C}$, 
%   we say that $(a, \langle m \rangle) \in \varpi(\Fe{}, 0)$ violates 
%   $c$ iff $\violate(\Fe{}, \langle m \rangle, c)$.
% \end{definition} 
% {\it Discussion: For the last two conditions, there can be at most one such 
% $(a, \langle m \rangle)$ in $\varpi(\Fe{}, 0)$ by definition. }
%\vspace{-0.5cm}
\subsection{Generalisation of standard complete labelling} 
Let us refine our earlier definition of standard complete labelling with 
 those constraints. 
% For a non-algorithmic characterisation (due to space,
% we do not introduce an algorithmic characterisation in this paper), it helps to keep 
% the following order on labelling that judges one labelling 
% strictly more informative than another when the former preserves 
% + and $-$ assignments by the latter and assigns a smaller number of ?. 
% \begin{definition}[Order in labelling]
%  Let $\rhd: \F \times \Lambda \times \Lambda$ be such that 
%  $(\Fe{}, \lambda_1, \lambda_2) \in \rhd$, written alternatively 
%  as $\lambda_1 \rhd^{\Fe{}} \lambda_2$ iff
%  $\lambda_2((a_p, \langle m_p \rangle)) \succeq \lambda_1((a_p, \langle m_p \rangle))$
%  for every
%  $(a_p, \langle m_p \rangle) \in \varpi(\Fe{}, 0)$. 
%  For any $\lambda_1, \lambda_2 \in \Lambda$ and for any $\Fe{}$, we say that 
%  $\lambda_1$ is more informative than $\lambda_2$ iff 
%  $\lambda_2 \rhd^{\Fe{}} \lambda_1$. In particular, 
%  $\lambda_1$ is strictly more informative than $\lambda_2$ 
%  iff $\lambda_2 \rhd^{\Fe{}} \lambda_1$ $\andC$ $\lambda_1 \not\rhd^{\Fe{}} \lambda_2$. 
%  \end{definition}
We define $\mathcal{C}$ to be 
$\{\textbf{G}, \textbf{S}, \star\}$, and refer to its subset 
by $C$ with or without a subscript. 
\begin{definition}[Complete labelling]
For any $\Fe{}$,  any $\lambda \in \Lambda$ and any 
 $C$, we say that 
   $\lambda$ is a complete labelling of $\Fe{}$ under $C$ iff
   every $(\{a_1, \ldots, a_n\}, \langle m \rangle) \in \varpi(\Fe{})$ satisfies
   all the following 
   conditions.

   \begin{itemize}[leftmargin=0.31cm]
       \item $\lambda((a_i, \langle m \rangle.n_i))
       = +$, $1 \leq i \leq n$, 
       iff both:
       \begin{itemize}
           \item every $a_j$, $ \leq j \leq n$, 
       attacking $a_i$ in $\ang{m}$ satisfies 
       $\lambda((a_j, \ang{m}.n_j)) = -$.
    
           \item if $\{\textbf{S}, 
           \star\} \in C$, 
           then assigning + to 
           $(a_i, \ang{m}.n_i)$ does not lead to 
       non-satisfaction of a semantic constraint $c \in C$ for any $(a, \ang{m_1}) \in \varpi(\Fe{})$.
    %   \item if $\star \in C$, 
    %       then assigning + to every $(a, \ang{m}.j)$ 
    %       with $\equalTo(a_i, a)$ does not 
    %       lead to non-satisfaction of a semantic constraint $c \in C$ 
    %       for any $(a_p, \ang{m_p}) \in \varpi(\Fe{}, 0)$.
       \end{itemize}
       
       \item $\lambda((a_i, \ang{m}.n_i)) = -$, 
       $1 \leq i \leq n$, iff both:
       \begin{itemize} 
          \item  there exists some $1 \leq j \leq n$ such that 
              $\lambda((a_j, \langle m \rangle.n_j)) = +$ 
              and that $a_j$ attacks $a_i$ in $\langle m \rangle$
              $\andC$ there is no $1 \leq k \leq n$ such that 
              $\lambda((a_k, \langle m \rangle.n_k)) = +$ 
              and that $a_k$ supports $a_i$ in $\langle m \rangle$.
              \item if $\{\textbf{S}, 
              \star\} \in C$, 
           then assigning $-$ to 
           $(a_i, \ang{m}.n_i)$ does not lead to 
       non-satisfaction of a semantic constraint $c \in C$ for any $(a, \ang{m_1}) \in \varpi(\Fe{})$.
    %   \item if $\star \in C$, 
    %       then assigning $-$ to every $(a, \ang{m}.j)$ 
    %       with $\equalTo(a_i, a)$ does not 
    %       lead to non-satisfaction of a semantic constraint $c \in C$ 
    %       for any $(a_p, \ang{m_p}) \in \varpi(\Fe{}, 0)$.
       \end{itemize} 
        \item  $\lambda((a_i, \ang{m}.n_i)) \not= +$,
        $1 \leq i \leq n$, 
   if $\textbf{G} \in C$ $\andC$ $(a_i, \ang{m}.n_i)$
   does not satisfy $\textbf{G}$.  
   \end{itemize}
  
\end{definition} 
Any argument that violates the graphic 
constraint will not be assigned + 
if $\textbf{G} \in C$. For both + and $-$, the first condition matches exactly 
the condition given for a standard complete labelling. The second condition
ensures conditions for a standard complete labelling to be maximally respected, in the sense 
that they normally apply unless by applying them 
there will be a non-satisfaction 
of a semantic constraint $c \in C$.

% semantic constraints to be respected, 
% in the following sense. No argument that violates the graphical constraint 
% would be assigned $+$, if $\textbf{G} \in C$, by a complete labelling. 
% no $(a, \ang{m}) \in \varpi(\Fe{}, 0)$ 
% which, if it is to be assigned + / $-$, would violate a semantic constraint in $C$ 
% is assigned + / $-$ by a complete labelling. Meanwhile, 
% $(a, \ang{m}) \in \varpi(\Fe{}, 0)$ 
% that is assigned $-$ or $?$ may be 
% violating a constraint, which is not 
% problematic {\it per se}. However, 
% in case $\star \in C$, it is possible 
% that $(a, \ang{m}) \in \varpi(\Fe{}, 0)$, 
% if it is to be assigned $-$, would cause 
% violation of $\star$ not by itself
% but collectively with other $(a, \ang{m_1})$. For example, 
% suppose $\Fe{} = (\{a_1, a_2, a_3, a_4, a_5\}, 
% \{(a_1, a_2), (a_3, a_4), (a_4, a_5), 
% (a_5, a_3)\}, \emptyset)$ such that
% $\equalTo(a_i, a_j)$ iff
% $(i, j) \in \{(2, 3), (3, 2)\}$. Suppose 
% also that $(a_2, 0.2), (a_3, 0.3) \in 
% \varpi(\Fe{}, 0)$.
% Then, 
% assigning $?$ to $(a_3, 0.3)$, 
% while $-$ to $(a_2, 0.2)$
% would violate $\star$. The 
% second condition for $-$ is thus necessary. \\
\indent Broader intuition is as follows. If a standard labelling of $\Fe{}$ satisfies 
all $c \in C$, then the standard complete labelling should itself be a complete labelling 
under $C$. Thus, when a standard complete labelling is not a complete labelling 
under $C$, it is because the labelling either induces violation of some semantic constraint $c \in C$ 
for some arguments occurring in $\Fe{}$ 
or + assignment to an argument that violates 
the graphic constraint if $\textbf{G} \in C$. In such cases, therefore, 
it will be required to make minimal 
change to the standard complete labelling so the
resulting labelling satisfies semantic 
constraints (if they are in $C$) and 
does not assign $+$ to an argument 
that violates $\textbf{G}$ (if 
it is in $C$). The second conditions for $+$ and $-$ ensure that the change
be indeed minimal with a complete labelling. 
% a complete labelling should be such that it 
% respects the definition of a standard complete 
% labelling maximally, in the sense that it is not possible to obtain 
% a strictly more informative labelling without either inducing semantic violation of some $c \in C$ or not assigning ? to an argument that violates
% the graphical constraint. 
% This motivates the last condition (Miminal loss). 
\begin{example}[Complete labelling] {\it (Continued)}\label{ex_4} 
  For the argumentation in \fbox{D} with the same assumptions and $\lambda_{1,2,3}$ 
  as in Example \ref{ex_3}, $\lambda_1$ is not a complete labelling under 
  $C$ if $\textbf{S} \in C$: $(a_0, 0.0)$ is assigned +, however,
  it does not satisfy $\textbf{S}$, so it should be assigned ?.
  $\lambda_2$ and $\lambda_3$ are a complete labelling under any $C$. 
  In addition, $\lambda_4$ such that 
  $\lambda_4((a_i, 0.i)) = ?, i \in \{0, 1\}$, 
  $\lambda_4((a_2, 0.2)) = \lambda_4((a_4, 0.4)) = +$, 
  $\lambda_4((a_3, 0.3)) = -$ is a complete labelling. $\lambda_4$ is not 
  a standard complete labelling. 

\end{example}
\begin{theorem}[Conservation]
   For any $\Fe{}$ and any $\lambda$, 
   if $C = \emptyset$, then $\lambda$ being a complete labelling 
   of $\Fe{}$ under $C$ is equivalent to $\lambda$ being a standard
   complete labelling of $\Fe{}$.
\end{theorem}
\begin{theorem}[No inclusion]
    For any $C \not= \emptyset$, there exists $\Fe{}$ such that some standard 
    complete labelling of $\Fe{}$ is not a complete labelling under $C$
    and that some complete labelling of $\Fe{}$ under $C$ is not a standard complete labelling.
\end{theorem}
This result holds even for $\Fe{}$ that does not involve any support (see Example \ref{ex_4}).
%\vspace{-0.1cm}
\subsection{Acceptability Semantics} 

\begin{definition}[Types of complete sets and acceptability semantics] 
    For any $\Fe{}$, we say that $A_1 \subseteq \textsf{Arg}(\Fe{}, 0)$ is: 
 complete under $C$ iff there exists a complete labelling $\lambda \in \Lambda$ of $\Fe{}$ under $C$ such that $A_1 = \{a \in \mathcal{A} \ | \ 
        \exists n \in \mathbb{N}\ \exists (a, 0.n) \in \varpi(\Fe{}).
        \lambda((a, 0.n)) = +\}$; 
    grounded under $C$ iff it is the set intersection of all complete sets under $C$; 
    preferred under $C$ iff it is a maximal complete set under $C$; and 
    semi-grounded iff it is a minimal 
    complete set. 
       
    We call the set of all complete / grounded / 
    semi-grounded / preferred sets under some $C$ complete / grounded / semi-grounded / preferred semantics under $C$. 
\end{definition} 
Note that ultimately we need to tell which subsets of
$\textsf{Arg}(\Fe{}, 0)$ are acceptable: this explains why we only look at $\langle m \rangle = 0$
for the semantics.
\begin{example}[Acceptability semantics] {\it (Continued)} 
For the argumentation in \fbox{D} with the same assumptions as in 
Example \ref{ex_3}, 
 if $\textbf{S} \not\in C$, then 
 we have the following semantics. 
 \vspace{-0.2cm} 
    \begin{itemize}
        \item[] complete:  $\{\{a_0\}, \{a_0, a_2, a_4\}, \{a_0, a_3\}\}$.
        \hspace{1.2cm} grounded:  $\{\{a_0\}\}$.
        \item[] semi-grounded: 
        $\{\{a_0\}\}$. \hspace{3.63cm} 
          preferred: $\{\{a_0, a_2, a_4\}, \{a_0, a_3\}\}$.  
    \end{itemize}
 If $\textbf{S} \in C$, then we have the following semantics.
 \vspace{-0.2cm} 
 \begin{itemize}
     \item[] complete:  $\{\{a_0\}, \{a_3\}, \{a_0, a_2, a_4\}\}$.
     \hspace{1.7cm} grounded: $\{\emptyset\}$.
     \item[] semi-grounded: 
     $\{\{a_0\}, \{a_3\}\}$\hspace{2.9cm} 
      preferred: $\{\{a_3\}, \{a_0, a_2, a_4\}\}$. 
 \end{itemize}
For the argumentation in \fbox{F} with the same assumptions as in Example~\ref{ex_3}, if 
$\star \in C$, then we have $\{\{a_5\}, \{a_9\}, \{a_5, a_9\}\}$ (complete), 
$\{\emptyset\}$ (grounded), $\{\{a_5\}, 
\{a_9\}\}$ (semi-grounded) and 
$\{\{a_5, a_9\}\}$ (preferred). 
\end{example} 

\begin{theorem}[Existence] 
For any $\Fe{}$ and any $C$, we have all the following:
\begin{enumerate}
    \item There exists at least one complete set under $C$.
    \item The grounded set may not be a complete set under $C$.
    \item A semi-grounded set is a complete set
    under $C$.
    \item If there is only one 
    semi-grounded set, then it is the grounded set
    under $C$.
\end{enumerate}
Additionally, if $\textbf{S}, \star \not\in C$, but not necessarily
otherwise, there is only one 
semi-grounded set under $C$.
\end{theorem} 
%\vspace{-0.5cm}
\section{Related Work and Conclusion}\label{sec:related}
%\vspace{-0.2cm}
\noindent \textbf{Structured Argumentation.}  
Structured argumentation \cite{Dung09,Modgil13} lets arguments
associated in premise-conclusion relation. Consider \ABA-style 
structured argumentation \cite{Dung09}:\linebreak $(\mathcal{H}, \mathcal{I}, \mathcal{J}, \mathcal{K})$ for: 
some finite set $\mathcal{H}$ of entities; $\mathcal{I}: \mathcal{H} \times 2^{\mathcal{H}}$; $\emptyset \not\subseteq \mathcal{J} \subseteq \mathcal{H}$; and 
$\mathcal{K}: \mathcal{J} \rightarrow \mathcal{H}$. $H_1 \subseteq \mathcal{H}$ is said to support $h_1 \in \mathcal{H}$ iff $(h_1, H_1) \in \mathcal{I}$. For any 
$(h, H) \in \mathcal{I}$, if $H \not= \emptyset$, then $h \in \mathcal{H} \backslash 
\mathcal{J}$. 
This tuple $(\mathcal{H}, \mathcal{I}, \mathcal{J}, \mathcal{K})$ ought to be such that the set $\mathcal{T}$ 
below exists. \\
\indent Let $\mathcal{T}: \mathcal{H} \times 2^{\mathcal{T}}$ be the least set satisfying both 
(1) $(h, \emptyset) \in \mathcal{T}$ if $h \in \mathcal{J}$, and (2) if 
$(h_1, t_1), \ldots, (h_n, t_n) \in \mathcal{T}$ and if $(h_0, \{h_1, \ldots, h_n\}) \in \mathcal{I}$, then 
$(h_0, \{(h_1, t_1), \ldots, (h_n, t_n)\})\linebreak \in \mathcal{T}$. 
Every $t \in \mathcal{T}$ is an argument in $\ABA$. 
$(h_1, t'_1) \in \mathcal{T}$ attacks $t_2 \in \mathcal{T}$ iff $\mathcal{K}(h) = h_1$
for some $h \in \mathcal{J}$ occurring in $t_2$. 
% Let $\tail: \mathcal{T} \rightarrow 2^H$ be such that $\tail((h_1, \{h_2\})) = \{h_2\}$; 
% $\tail((h_1, \emptyset)) = \emptyset$; else, 
% $\tail((h, \{t_1, \ldots, t_n\})) = \tail(t_1) \cup \cdots \cup \tail(t_n)$. 
% If $t \in \mathcal{T}$ and if $\tail(t) = H_1$, then 
% $t$ that replaces every $h \in H_1$ with any $I$ is also in $\mathcal{T}$.
% Then $T \subseteq \mathcal{T}$ such that no $t \in T$ is replaceable and that
% $\tail(t) = \emptyset$ stores 
% $\ABA$ arguments. An $\ABA$ argument $a_1$ attacks another argument $a_2$ just when 
% there exists some $h_2 \in J$ that occurs in $a_2$ is such that $K(h_2) = \head(a_1)$.
$T_1 \subseteq \mathcal{T}$ defends $t \in \mathcal{T}$ just when
every $t_p \in \mathcal{T}$ attacking $t$ is attacked by at least some $t_q \in T_1$. 
It is complete just when: (1) there is no attack among the members of $T_1$; and (2) 
it includes all $t \in \mathcal{T}$ it defends.

Encoding of $\ABA$ into $\Fe{} = (A, R, \Rs)$ is rather smooth.  
We define $\mathcal{X}$ to be such that $x_{h} \in\mathcal{X}$ iff
$h \in \mathcal{J}$.  
We define $A$ to be the least set that satisfies both:
(1) $a_h = (\{x_h\}, \emptyset, \emptyset)$ is in $A$ if
$(h, \emptyset) \in \mathcal{T}$ for some $h$; and (2) 
$a_{t} = (\{a_{T_1}, a_{h}\}, \emptyset, \{(a_{T_1}, a_h)\})$ such that $a_{T_1} = (\{a_{t_1}, \ldots, a_{t_n}\}, \emptyset, \emptyset)$ is in $A$ if
$T_1 = \{t_1, \ldots, t_n\}$ for some $T_1 \subseteq \mathcal{T}$ and some $h$ such that 
$(h, T_1) \in \mathcal{T}$. 
We define $R$ to be such that $(a_1, a_2) \in R$ for some $a_1, a_2 \in A$ 
iff, if $a_1$ ($a_2$) is the encoding of $t_1 \in \mathcal{T}$ ($t_2 \in \mathcal{T}$), 
then $t_1$ attacks $t_2$. Finally, we define $\Rs$ to be empty. 
Then $A_1 \subseteq A$ is a complete set ($\langle m \rangle = 0$) under $C = 
\emptyset$ iff 
$T_1 \subseteq \mathcal{T}$ such that $A_1$ is the encoding of $T_1$ is complete. Trivially, 
      $\ABA$ is an instance of $\BBA$. 
      
More than just the result of subsumption, the process of encoding highlights 
the degree of structuralisation allowed in $\BBA$ and $\ABA$, which is similarly true of $\ASPIC$ \cite{Modgil13}: (1) support (premise-conclusion) in $\ABA$ is accessible for a child (premise) of a subtree of $t \in \mathcal{T}$, but 
no support, if it should occur in the parent (conclusion) of a subtree of $\mathcal{T}$, is accessible. That is, for the 
argument: [something is a support for [that something is a support for something]], 
the outer support is accessible in $\ABA$ formalisation, but the inner support, 
within the conclusion, 
must be encapsulated. $\BBA$ relaxes this restriction. Similarly, 
(2) any construct of the sort: [[that something attacks something] supports something] 
is expressed in $\ABA$ by encapsulating the inner attack. In $\BBA$, the inner attack is also accessible. 
The concept of block argumentation thus allows for a greater reach to the internal 
structure of arguments. It is our hope, then, that the study in this paper will be of interest also to 
the scientific community working on structured argumentation.  

\vspace{0.15cm}
\noindent \textbf{Dependency among Arguments.}  
Dung abstract argumentation does not specify 
the nature of an abstract entity. It is possible two arguments in an argumentation graph are interpreted 
identical. For rhetoric argumentation, 
acceptance of an argument may be considered with respect to the position in the graph 
in which it occurs; thus such a scenario incurs comparatively small an issue. 
Outside rhetoric argumentation, however, acceptability statuses of arguments 
are often preferred to be regarded indicative of acceptance of the arguments and not 
their acceptance with respect to their positions in the graph. It is then 
that multiple occurrences of the same argument in a graph are less desirable. An extreme
case is ``$a_1$ attacks $a_2$''. In case they are interpreted the same, classic semantics predict the argument to be both acceptable and not acceptable, and yet it sceptically accepts it, i.e. 
$a_1$. 
Such issue from dependencies seems to have been seldom reflected back to 
abstract argumentation semantics. In abstract dialectical frameworks 
\cite{Brewka10}, acceptability statuses of 
arguments are given based on those of 
their attackers/supporters. In this work, we showed 
that practically any arguments in a graph, whether 
or not they are connected, may have dependencies. We proposed use of both 
graphical and semantic constraints. 

\vspace{0.15cm}
\noindent \textbf{Higher-Level and Meta Argumentation.}
Meta argumentation \cite{Modgil11,Villata10} facilitates layers of argumentation, to discuss 
attack relation among arguments and so on in a preceding layer. From a layer 
to a layer, there is a clear boundary (and the separation does not 
disappear if one is to flatten them). Our motivating example in Section~\ref{sec:motivation} shows, 
however, that [something attacks something] may itself be an argument, which $\BBA$ can handle
uniformly as with unitary ones. Temporal / modal argumentation networks \cite{Barringer12}, as far as we are aware, is the first abstract 
argumentation study 
that hinted at the possibility that an unattacked argument may not be outright accepted. A 
temporal argument $\diamond a$ asks effectively if $a$ is accepted in the argumentation 
at a connected possible 
world for the argument to be accepted-able. The temporal argument allows 
 an argument in another possible world to be referred to. In this work, 
 we have generalised such semantic dependency 
 with the semantic constraints. 
Higher-level argumentation
\cite{Gabbay09b} considers substitution of an argumentation framework into 
an argument to identify which arguments 
in the substituted argumentation are the actual interactors to the outside. 
By contrast, in block argumentation, the key is to express the dual roles of an argument: 
as an argument and as an argumentation (similar emphasis 
given for coalition formation \cite{Arisaka17a}) and how they influence 
acceptability semantics. In block 
argumentation, it is possible that an argument
as an argumentation, 
and not necessarily some arguments in the 
argumentation, attacks or supports other 
arguments. 

\vspace{0.2cm}
\noindent \textbf{Conclusion.}
We presented block argumentation, with comparisons 
to structured argumentation for which we have 
demonstrated a prospect of further generalisation. Block argumentation 
reveals multiple occurrences of the same arguments within an argumentation, and, as we have identified, both graphical and semantic
coherency pose challenges to the classic principle of always 
accepting unattacked arguments. We proposed to tackle this issue with constraints. For future work, we plan to study 
their granularity to cater 
for specific applications, as well as to incorporate probabilistic or dynamic approaches. 

\section*{Acknowledgement}
The first author thanks Jiraporn Pooksook for a tutorial on $\ABA$. 
% Curiously enough, they led to recovery of symmetry between 
% minimum and maximum 
% semantics, that just as there may not be the greatest complete set, there may not be
% the least complete set; and just as there is always a maximal complete set, 
% so is there a minimal complete set. 

% \todo[inline]{Hi Ryuta, thanks for your work. The paper is very interesting and solid as well! I went through it and fixed a few things. I'll do it again before the final submission. I have  added the institution back, since the paper has to be explicitly non-anonymous and I would like to be completely clear about the authors. I added a few comments in orange for you: please have a look at them. I only have two ``major'' concerns. i) it would be better to refer to an additional document with the proofs of theorems (document online); the prevents a reviewer to say that the paper is not technically unsound. ii) the last section maybe needs some more conclusions and a bit of future work since it is unbiased towards related work. Maybe it is a bit too much dense. We can work on it!}

\bibliographystyle{splncs04}
\bibliography{references} 

\begin{thebibliography}{10}
\providecommand{\url}[1]{\texttt{#1}}
\providecommand{\urlprefix}{URL }
\providecommand{\doi}[1]{https://doi.org/#1}

\bibitem{Arisaka16d}
Arisaka, R., Satoh, K.: {Voluntary Manslaughter? A Case Study with
  Meta-Argumentation with Supports}. In: {JSAI-isAI Workshops}. pp. 241--252
  (2016)

\bibitem{Arisaka17a}
Arisaka, R., Satoh, K.: {Coalition Formability Semantics with
  Conflict-Eliminable Sets of Arguments}. In: {AAMAS}. pp. 1469--1471 (2017)

\bibitem{Barringer12}
Barringer, H., Gabbay, D.M.: {Modal and Temporal Argumentation Networks}.
  {Argument \& Computation}  \textbf{3}(2-3),  203--227 (2012)

\bibitem{bistarelli18}
Bistarelli, S., Rossi, F., Santini, F.: {A novel weighted defence and its
  relaxation in abstract argumentation}. International Journal of Approximate
  Reasoning  \textbf{92},  66--86 (2018)

\bibitem{lpnmr16}
Bistarelli, S., Santini, F.: {A Hasse Diagram for Weighted Sceptical Semantics
  with a Unique-Status Grounded Semantics}. In: LPNMR. {LNCS}, vol. 10377, pp.
  49--56. Springer (2017)

\bibitem{Boella10}
Boella, G., Gabbay, D.M., van~der Torre, L., Villata, S.: {Support in Abstract
  Argumentation}. In: {COMMA}. pp. 111--122 (2010)

\bibitem{Brewka10}
Brewka, G., Woltran, S.: {Abstract Dialectical Franeworks}. In: {KR}. pp.
  102--111 (2010)

\bibitem{Caminada06}
Caminada, M.: {On the Issue of Reinstatement in Argumentation}. In: {JELIA}.
  pp. 111--123 (2006)

\bibitem{Cayrol11}
Cayrol, C., Lagasquie-Schiex, M.C.: {Bipolarity in Argumentation Graphs:
  Towards a Better Understanding}. In: {SUM}. pp. 137--148 (2011)

\bibitem{Dung95}
Dung, P.M.: On the {Acceptability} of {Arguments} and {Its} {Fundamental}
  {Role} in {Nonmonotonic} {Reasoning}, {Logic Programming}, and n-{Person}
  {Games}. Artificial {Intelligence}  \textbf{77}(2),  321--357 (1995)

\bibitem{Dung09}
Dung, P.M., Kowalski, R.A., Toni, F.: {Assumption-based argumentation}. In:
  {Argumentation in Artificial Intelligence}, pp. 199 -- 218. Springer (2009)

\bibitem{Gabbay09b}
Gabbay, D.M.: {Semantics for Higher Level Attacks in Extended Argumentation
  Frames Part 1: Overview}. {Studia Logica}  \textbf{93}(2-3),  357--381 (2009)

\bibitem{Jech2000}
Jech, T.: {SET THEORY}. Springer, 3rd edn. (2000)

\bibitem{Kleene52}
Kleene, S.C.: Introduction to META-MATHEMATICS. North-Holland Publishing Co.
  (1952)

\bibitem{Modgil11}
Modgil, S., Bench-Capon, T.J.M.: {Metalevel argumentation}. {Journal of Logic
  and Computation}  \textbf{21}(6),  959--1003 (2011)

\bibitem{Modgil13}
Modgil, S., Prakken, H.: {A general account of argumentation with preferences}.
  {Artificial Intelligence}  \textbf{195},  361--397 (2013)

\bibitem{Nouioua10}
Nouioua, F., Risch, V.: {Bipolar Argumentation Frameworks with Specialized
  Supports}. In: {ICTAI}. pp. 215--218 (2010)

\bibitem{Nouioua11}
Nouioua, F., Risch, V.: {Argumentation Frameworks with Necessities}. In: {SUM}.
  pp. 163--176 (2011)

\bibitem{Oren10}
Oren, N., Reed, C., Luck, M.: {Moving between Argumentation Frameworks}. In:
  {COMMA}. pp. 379--390 (2010)

\bibitem{hunter18}
Polberg, S., Hunter, A.: Empirical evaluation of abstract argumentation:
  Supporting the need for bipolar and probabilistic approaches. International
  Journal of Approximate Reasoning  \textbf{93},  487--543 (2018)

\bibitem{Slife95}
Slife, B.D., Williams, R.N.: What's Behind the Research? Discovering Hidden
  Assumptions in the Behavioral Sciences. Sage Publications (1995)

\bibitem{Villata10}
Villata, S., Boella, G., Gabbay, D.M., van~der Torre, L.: Arguing about {Trust}
  in {Multiagent} {Systems}. In: AAAI. pp. 236--243. AAAI Press (2010)

\end{thebibliography}

%\pagebreak
\section*{Appendix (Proofs)}
\setcounter{theorem}{0}
\begin{theorem}[Correspondence between extension-based complete sets and complete labellings]
   For any $\Fe{}$ and any $\langle m \rangle$, 
   $A_1 \subseteq \textsf{Arg}(\Fe{}, \langle m \rangle)$ is standard complete in $\langle m \rangle$
   only if there is some standard complete labelling $\lambda$ of $\Fe{}$ 
   such that
   $\lambda((a_p, \langle m \rangle.n)) = +$  is equivalent to 
   $a_p \in A_1$ for any $(a_p, \langle m \rangle.n) \in  \varpi(\Fe{})$. Conversely, $\lambda$ is a standard complete labelling of $\Fe{}$ 
   only if, for every $(a, \langle m \rangle) \in \varpi(\Fe{})$,
  {\small $\{a_p \in \textsf{Arg}(\Fe{}, \langle m \rangle) \ | \ 
  \exists n \in \mathbb{N}.\lambda((a_p, \langle m \rangle.n)) = +\}$} is a standard complete set in $\langle m \rangle$.  
\end{theorem}
\begin{proof} 
  Assume $A_1 \subseteq \textsf{Arg}(\Fe{}, \langle m \rangle)$ is
  standard complete in $\langle m \rangle$. 
  Assume some $\lambda \in \Lambda$ such that 
  $\lambda((a_p, \langle m \rangle.n)) = +$ for any
  $(a_p, \langle m \rangle.n) \in \varpi(\Fe{})$
  with $a_p \in A_1$, that 
  $\lambda((a_p, \langle m \rangle.n)) = -$ for any 
  $(a_p, \langle m \rangle.n) \in \varpi(\Fe{})$ 
  with $a_p \in \textsf{Arg}(\Fe{}, \ang{m})$ if $a_p$ is attacked  
  by $A_1$, and that 
  $\lambda((a_p, \langle m \rangle.n)) =\ ?$ for any other
  $(a_p, \langle m \rangle.n) \in \varpi(\Fe{})$ 
  with $a_p \in \textsf{Arg}(\Fe{}, \langle m \rangle)$. 
  We show that $\lambda$ is a standard complete labelling. 
  We show firstly that every $(a_p, \langle m \rangle.n) \in \varpi(\Fe{})$ such that $\lambda((a_p, \langle m \rangle.n))
  = +$ satisfies the corresponding condition (Definition 6). 
  Suppose, by way of showing contradiction, that there is 
  some $(a_p, \langle m \rangle.n) \in \varpi(\Fe{})$ such that $\lambda((a_p, \langle m \rangle.n))
  = +$ does not satisfy the corresponding condition. Then, 
  there must exist some $(a_1, \ang{m}.n_1) \in \varpi(\Fe{}))$
  such that $\lambda((a_1, \ang{m}.n_1)) = +$ and that $a_1$ 
  attacks $a_p$ in $\ang{m}$. 
  However, it cannot be $+$, since $A_1$ is conflict-free, contradiction. Second, we 
  show that 
  every $(a_p, \langle m \rangle.n) \in \varpi(\Fe{})$ such that $\lambda((a_p, \langle m \rangle.n))
  = -$ satisfies the corresponding condition (Definition 6). By construction of 
  $\lambda$, however, $a_p$ is attacked 
  by a member of $A_1$. Since $A_1$ is defended,
  there exists no member of $A_1$ that supports
  $a_p$. Finally, we show that every $(a_p, \langle m \rangle.n) \in \varpi(\Fe{})$ such that $\lambda((a_p, \langle m \rangle.n))
  =\ ?$ does not satisfy the conditions
  for $+$ or $-$ (Definition 6). 
  If it satisfied the condition for $+$, 
  then every $a_1$ attacking it 
  would get $-$. By construction of 
  $\lambda$, however, 
  every such $a_1$ would be attacked by a member of $A_1$ in $\ang{m}$. Since $A_1$ includes all arguments 
  it defends in $\ang{m}$, it follows 
  that $a_p \in A_1$, contradiction. 
  If it satisfied the condition for $-$, 
  then it would be attacked by 
  a member of $A_1$ in $\ang{m}$. By construction
  of $\lambda$, it is not possible that 
  $a_p$ gets $?$. 
  \\\\
%   and so it has to be $?$. 
%   This means that: 
%   \begin{enumerate}
%       \item There exists some 
%       $(a_2, \ang{m}.2) \in 
%       \varpi(\Fe{}, 0)$ such that: (1) 
%       $a_2$ attacks $a_1$ in $\ang{m}$; 
%       and (2) $\lambda((a_2, \ang{m}.2)) 
%       \in \{+, ?\}$. (Otherwise, $a_1$ would get 
%       +). 
%       \item Either: 
%          \begin{enumerate}
%              \item there is no  
%       $(a_2, \ang{m}.2) \in 
%       \varpi(\Fe{}, 0)$ such that: (1) 
%       $a_2$ attacks $a_1$ in $\ang{m}$; 
%       and (2) $\lambda((a_2, \ang{m}.2)) = +$; or 
%       \item there exists some 
%         $(a_2, \ang{m}.2) \in 
%       \varpi(\Fe{}, 0)$ such that: (1) 
%       $a_2$ supports $a_1$ in $\ang{m}$; 
%       and (2) $\lambda((a_2, \ang{m}.2)) = +$.
%          \end{enumerate}
%          (Otherwise, $a_1$ would get $-$).
%   \end{enumerate}
%   Now, 
%   since $A_1$ is defended, $a_1$ is attacked 
%   by a member of $A_1$, thus 
%   2-(a) is not satisfiable. Hence 
%   by 2-(b) there exists some 
%   member of $A_1$ that supports $a_1$. 
%   This contradicts the assumption that 
%   $A_1$ is defended. \\\\
  \indent For the second part, 
  assume $\lambda$ is a standard complete labelling. Assume also some $\ang{m}$ 
  such that 
  $\textsf{Arg}(\Fe{}, \ang{m})$ is non-empty. 
  Assume $A_1 = \{a_p \in \textsf{Arg}(\Fe{}, \langle m \rangle) \ | \ 
  \exists n \in \mathbb{N}.\lambda((a_p, \langle m \rangle.n)) = +\}$. We show that 
  $A_1$ is standard complete. 
  $A_1$ is trivially conflict-free. 
  Now, suppose some $a_1 \in \textsf{Arg}(\Fe{}, 
  \ang{m})$ attacking a member of $A_1$. 
  Since $\lambda$ is a standard complete labelling,
  it must be that $a_1$ gets $-$ in $\ang{m}$, 
  which means that $A_1$ attacks but does not support $a_1$ in $\ang{m}$. Thus $A_1$ defends 
  its members attacked by $a_1$. 
  By the definition of $\lambda$, it is 
  straightforward to see that 
  $A_1$ includes all arguments it defends. \qed
  \end{proof} 
\begin{theorem}[Conservation]
   For any $\Fe{}$ and any $\lambda$, 
   if $C = \emptyset$, then $\lambda$ being a complete labelling 
   of $\Fe{}$ under $C$ is equivalent to $\lambda$ being a standard
   complete labelling of $\Fe{}$.
\end{theorem}
\begin{proof}
   Trivial. \qed
\end{proof}
\begin{theorem}[No inclusion]
    For any $C \not= \emptyset$, there exists $\Fe{}$ such that some standard 
    complete labelling of $\Fe{}$ is not a complete labelling under $C$
    and that some complete labelling of $\Fe{}$ under $C$ is not a standard complete labelling.
\end{theorem}
\begin{proof}
   See Example 6 when $C$ contains 
   $\textbf{S}$ or $\star$. For $C = \{\textbf{G}\}$, consider the argumentation 
   in \fbox{B} as $\Fe{}$.   \qed
\end{proof}
For the last theorem, it helps to define order on labelling that judges one labelling 
strictly more informative than another when the former preserves 
+ and $-$ assignments by the latter but assigns a smaller number of ?. 
\begin{definition}[Order in labelling]
 Let $\rhd: \F \times \Lambda \times \Lambda$ be such that 
 $(\Fe{}, \lambda_1, \lambda_2) \in \rhd$, written alternatively 
 as $\lambda_1 \rhd^{\Fe{}} \lambda_2$ iff
 $\lambda_2((a_p, \langle m_p \rangle)) \succeq \lambda_1((a_p, \langle m_p \rangle))$
 for every
 $(a_p, \langle m_p \rangle) \in \varpi(\Fe{})$. 
 For any $\lambda_1, \lambda_2 \in \Lambda$ and for any $\Fe{}$, we say that 
 $\lambda_1$ is more informative than $\lambda_2$ iff 
 $\lambda_2 \rhd^{\Fe{}} \lambda_1$. In particular, 
 $\lambda_1$ is strictly more informative than $\lambda_2$ 
 iff $\lambda_2 \rhd^{\Fe{}} \lambda_1$ $\andC$ $\lambda_1 \not\rhd^{\Fe{}} \lambda_2$. 
 \end{definition}
 We say, further, that there is a path from $a_1$ to $a_2$ in $\ang{m}$ 
 iff $a_1, a_2 \in \textsf{Arg}(\Fe{}, \ang{m})$ $\andC$ 
 either: $a_1$ attacks $a_2$; or else there is a path from 
 $a_1$ to some $a_3$ such that $a_3$ attacks $a_2$.
 
 \begin{lemma}
    Let $\Fe{}$ be a member of $\F$, and 
    let $\lambda_1$ be 
    a standard complete labelling 
    of $\Fe{}$. 
    Let $\Gamma$ be 
    a set of 
    some $(a_1, \ang{m_1}), 
    \ldots, (a_n, \ang{m_n}) \in \varpi(\Fe{})$
    such that 
    $\lambda_1((a_i, \ang{m_i})) \not=\ ?$.
     Let $\Lambda_1$ be 
     a subclass of $\Lambda$ 
     that contains all
     $\lambda$ such that 
     $\lambda((a_i, \langle m_i \rangle)) = \ ?$
     for any $(a_i, \ang{m_i}) \in \Gamma$ 
     but, for any 
     other $(a_p, \ang{m_p}) \in 
     \varpi(\Fe{})$, 
     that $\lambda$ respects the conditions 
     given for a standard complete labelling. 
     Then, there exists $\lambda \in \Lambda_1$ 
     such that $\lambda_1$ 
     is strictly more informative 
     than $\lambda$.
 \end{lemma} 
\begin{theorem}[Existence] 
For any $\Fe{}$ and any $C$, we have all the following:
\begin{enumerate}
    \item There exists at least one complete set under $C$.
    \item The grounded set may not be a complete set under $C$.
    \item A semi-grounded set is a complete set
    under $C$.
    \item If there is only one unique 
    semi-grounded set, then it is the grounded set
    under $C$.
\end{enumerate}
Additionally, if $\textbf{S}, \star \not\in C$, but not necessarily
otherwise, there is only one 
semi-grounded set under $C$.
\end{theorem} 
\begin{proof}
 2. - 4. follow immediately from the preceding 
 discussion, on the assumption that 1. holds to be
 the case. Suppose, by way of showing contradiction, 
 that there is no complete labelling. Then, it must be that 
 there exist some $\Fe{}$ and some $C$ such that every $\lambda \in \Lambda$ of $\Fe{}$ contradicts 
 at least one condition given in Definition 12 for some $(a, \ang{m}) \in \varpi(\Fe{})$. Consider some standard complete labelling 
 $\lambda$ which always exists. By the supposition, $\lambda$ is not a complete labelling. 
 Therefore, there exists some $(a, \ang{m}) \in \varpi(\Fe{})$ that either does not 
 satisfy some semantic constraint $c \in C$ 
 with $\lambda$, 
 or violates $\textbf{G}$ and yet 
 is assigned + with $\lambda$. Let us define the following function.\\
 
 \noindent \textbf{Function} \textsf{down} $(\Fe{}:\F, \lambda: \Lambda, C:\mathcal{C})$ returns $(\Fe{}:\F, \lambda_1:\Lambda, C:\mathcal{C})$.\\\\
 Assume $(a, \ang{m}) \in \varpi(\Fe{})$ 
 for which $\lambda \in \Lambda$ of $\Fe{}$ contradicts 
 at least one condition given in Definition 12.
 \begin{enumerate}
    \item  \textbf{return} $(\Fe{}, \lambda, C)$ 
    if there is no such $(a, \ang{m})$.

     \item If $\lambda((a, \ang{m_p}.n)) \in \{+, -\}$ ($\ang{m} = \ang{m_p}.n$), $\star \in C$, 
     and $(a, \ang{m_p}.n)$ does not satisfy $\star$, then there exist 
     $(a_1, \ang{m_p}.n_1), \ldots, (a_k, \ang{m_p}.n_k) \in \varpi(\Fe{})$ 
     such that $\bigwedge_{1 \leq i \leq k}\equalTo(a, a_i)$ and that 
     $|\bigcup_{1 \leq i \leq k}\lambda((a_i, \ang{m_p}.n_i))| > 1$. Consider 
     a subclass $\Lambda_1$ of $\Lambda$ 
     that contains all $\lambda_1$
     such that $\lambda_1((a, \ang{m_p}.n)) = 
     \lambda_1((a_1, \ang{m_p}.n_1)) = \cdots = 
     \lambda_1((a_k, \ang{m_p}.n_k)) = \ ?$, 
     but which otherwise respects the conditions given for a standard complete labelling. 
     By Lemma 1, and by the fact that 
     there is at least one \linebreak
     $(a_i, \ang{m_p}.n_i)$, $1 \leq i \leq k$, 
     that is not assigned ?, 
     there exists some $\lambda_1 \in \Lambda_1$
     such that $\lambda$ is strictly 
     more informative than $\lambda_1$. 
     By construction, such $\lambda_1$ always exists. \textbf{return} 
     $(\Fe{}, \lambda_1, C)$.

     \item If $\lambda((a, \ang{m})) \in \{+, -\}$, $\textbf{S} \in C$,
     and $(a, \ang{m})$ does not satisfy $\textbf{S}$: 
     Consider 
     a subclass $\Lambda_1$ of $\Lambda$ 
     that contains all $\lambda_1$
     such that $\lambda_1((a, \ang{m})) = 
     \ ?$, 
     but which otherwise respects the conditions given for a standard complete labelling. 
     By Lemma 1, and by the fact that 
     $\lambda((a, \ang{m})) \not= \ ?$, 
     there exists some $\lambda_1 \in \Lambda_1$
     such that 
     $\lambda$ is strictly more informative
     than $\lambda_1$. By construction, 
     such $\lambda_1$ always exists. 
     \textbf{return} 
     $(\Fe{}, \lambda_1, C)$.

   \item If $\lambda((a, \ang{m})) = +$, $\textbf{G} \in C$,
     and $(a, \ang{m})$ does not satisfy $\textbf{G}$: 
      Consider 
     a subclass $\Lambda_1$ of $\Lambda$ 
     that contains all $\lambda_1$
     such that $\lambda_1((a, \ang{m})) = 
     \ ?$, 
     but which otherwise respects the conditions given for a standard complete labelling. 
     By Lemma 1, and by the fact 
     that $\lambda((a, \ang{m})) \not=\ ?$, 
     there exists some $\lambda_1 \in \Lambda_1$ 
     such that 
     $\lambda$ is strictly more informative 
     than $\lambda_1$. By construction, 
     such $\lambda_1$ always exists. 
     \textbf{return} $(\Fe{}, \lambda_1, C)$. 
 \end{enumerate}
 \noindent \textbf{End Function} \\\\
 Let $(\Fe{}, \lambda_1, C)$ be: $(\Fe{}, \lambda_1, C) = \textsf{down}^{n+1}(\Fe{}, \lambda, C) = 
 \textsf{down}^n(\Fe{}, \lambda, C)$ for some $n$ and for some standard 
 complete labelling $\lambda$. For every 
 $(a, \ang{m}) \in \varpi(\Fe{})$, it holds that 
 $(a, \ang{m})$ and $\lambda_1$ satisfy semantic constraints in $C$, 
 and that $\lambda((a, \ang{m})) \not= +$ if $(a, \ang{m})$ does not satisfy $\textbf{G}$ if $\textbf{G} 
 \in C$.\\
 \indent Now, consider $\Lambda_1 \subseteq \Lambda$ to be the set of 
 all labelling $\lambda_2$ such that $\lambda_2$ is more 
 informative than $\lambda_1$ but not 
 more informative than $\lambda$, and 
 moreover that, for every 
 $(a, \ang{m}) \in \varpi(\Fe{})$, it holds that 
 $(a, \ang{m})$ and $\lambda_1$ satisfy semantic constraints in $C$, 
 and that $\lambda((a, \ang{m})) \not= +$ if $(a, \ang{m})$ does not satisfy $\textbf{G}$ if $\textbf{G} 
 \in C$. Since 
 both $\lambda_1$ and $\lambda$ exist, $\Lambda_1 \not= \emptyset$. It is easy to see 
 that a maximal element in $\Lambda_1$ is a complete labelling of $\Fe{}$. \\\\
\indent If $C =\textbf{G}$ or else $C = \emptyset$, 
 then, for any $(a, \ang{m}) \in \varpi(\Fe{})$ 
 if 
 there is some $\lambda$ that is strictly less informative than a standard complete labelling $\lambda_1$ of $\Fe{}$
 and if $\lambda((a, \ang{m})) \succeq \lambda_1((a, \ang{m}))$, then 
 for any  complete labelling $\lambda_2$ of $\Fe{}$ under $C$, 
 $\lambda_2((a, \ang{m})) = \lambda((a, \ang{m}))$. Straightforward. Example 6 for $\textbf{S} \in C$ and 
 $\star \in C$. \qed
\end{proof}
\end{document}